\documentclass[12pt]{article}
\usepackage{graphicx} 
\usepackage{float}    
\usepackage{verbatim} 
\usepackage{amsmath}  
\usepackage{amssymb}  
\usepackage{subfig}   
\usepackage[colorlinks=true,citecolor=blue,linkcolor=blue,urlcolor=blue]{hyperref}
\usepackage{fullpage}
\usepackage{amsthm,mathtools}
\usepackage{enumerate}
\usepackage{paralist}
\usepackage{xspace}
\usepackage{caption}
\usepackage{bbm}
\usepackage{url}
\usepackage{mathrsfs}
\usepackage{algorithm}
\usepackage{algorithmic}

\newcommand{\vertiii}[1]{{\left\vert\kern-0.25ex\left\vert\kern-0.25ex\left\vert #1 
		\right\vert\kern-0.25ex\right\vert\kern-0.25ex\right\vert}}

\makeatletter

\newtheorem{theorem}{Theorem}

\newtheorem{lemma}{Lemma}

\newtheorem{corollary}{Corollary}
\newtheorem{proposition}{Proposition}

\usepackage[margin=1in]{geometry}
\usepackage{times}

\usepackage{thmtools}
\usepackage{thm-restate}

\usepackage{listings}

\@ifundefined{showcaptionsetup}{}{%
	\PassOptionsToPackage{caption=false}{subfig}}
\usepackage{subfig}
\makeatother

\usepackage{listings}

\global\long\def\H2{\mathcal{H}_2}

\global\long\def\E1{\mathcal{E}_1}

\global\long\def\Csys{C_{\mathrm{sys}}}

\begin{document}

	\title{\bf Learning Partially Observed Linear Dynamical Systems from Logarithmic Number of Samples}
	\author{Salar Fattahi
		\date{%
		University of Michigan	
		}
	}

	\maketitle
	
	\begin{abstract}
		In this work, we study the problem of learning partially observed linear dynamical systems from a single sample trajectory. A major practical challenge in the existing system identification methods is the undesirable dependency of their required sample size on the system dimension: roughly speaking, they presume and rely on sample sizes that scale linearly with respect to the system dimension. Evidently, in high-dimensional regime where the system dimension is large, it may be costly, if not impossible, to collect as many samples from the unknown system. In this paper, we will remedy this undesirable dependency on the system dimension by introducing an $\ell_1$-regularized estimation method that can accurately estimate the Markov parameters of the system, provided that the number of samples scale logarithmically with the system dimension. Our result significantly improves the sample complexity of learning partially observed linear dynamical systems: it shows that the Markov parameters of the system can be learned in the high-dimensional setting, where the number of samples is significantly smaller than the system dimension. Traditionally, the $\ell_1$-regularized estimators have been used to promote sparsity in the estimated parameters. By resorting to the notion of ``weak sparsity'', we show that, \textit{irrespective of the true sparsity of the system}, a similar regularized estimator can be used to reduce the sample complexity of learning partially observed linear systems, provided that the true system is inherently stable.
	\end{abstract}

\section{Introduction}\label{sec:intro}

Most of today's real-world systems are characterized by being large-scale, complex, and safety-critical. For instance, the nation-wide power grid is comprised of millions of active devices that interact according to uncertain dynamics and complex laws of physics~\cite{blaabjerg2006overview, amin2005toward, wang1998robust}. 
As another example, the contemporary transportation systems are moving towards a spatially distributed, autonomous, and intelligent infrastructure with thousands of heterogeneous and dynamic components~\cite{barbaresso2014usdot, krechmer2018effects}. Other examples include aerospace systems~\cite{kapila2000spacecraft}, decentralized wireless networks~\cite{kubisch2003distributed}, and multi-agent robot networks~\cite{nguyen2011model}. A common feature of these systems is that they are comprised of a massive network of interconnected subsystems with complex and uncertain dynamics.

The unknown structure of the dynamics on the one hand, and the emergence of machine learning and reinforcement learning (RL) as powerful tools for solving sequential decision making problems~\cite{sutton2018reinforcement, krizhevsky2012imagenet, duan2016benchmarking} on the other hand, strongly motivate the use of \textit{data-driven} methods in the operation of unknown safety-critical systems. However, the applications of machine learning techniques in the {safety-critical} systems remain mostly limited due to several fundamental challenges. First, to alleviate the so-called ``curse of dimensionality'' in these systems, any practical learning and control method must be data-, time-, and memory-efficient. Second, rather than being treated as ``black-box'' models, these systems must be governed via models that are interpretable by practitioners, and are amenable to well-established robust/optimal control methods. 

With the goal of addressing the aforementioned challenges, this paper studies the efficient learning of partially observed linear systems from a single trajectory of input-output measurements. Despite a mature body of literature on the statistical learning and control of linear dynamical systems, their practicality remains limited for large-scale and safety-critical systems. A key challenge lies in the required sample sizes of these methods and their dependency on the system dimensions: 
for a system with dimension $n$, the best existing system identification techniques require sample sizes in the order of $\mathcal{O}(n)$ to $\mathcal{O}(n^4)$ to provide certifiable guarantees on their
performance~\cite{oymak2019non, krauth2019finite, dean2019sample, simchowitz2019learning, sarkar2019finite}. Such dependency may inevitably lead to exceedingly long interactions with the safety-critical system, where it is extremely costly or even impossible to collect nearly as many samples without jeopardizing its safety---consider sampling from a geographically distributed power grid with tens of millions of parameters, and this increasing difficulty becomes apparent.\vspace{2mm}

\noindent{\bf Contributions:} In this work, we show that the Markov parameters defining the input-output behavior of partially observed linear dynamical systems can be learned with logarithmic sample complexity, i.e., from a single sample trajectory whose length scales poly-logarithmically with the output dimension. Our result relies on the key assumption that the system is inherently stable, or alternatively, it is equipped with an initial stabilizing controller. We show that the inherent stability of the system is analogous to the notion of \textit{weak sparsity} in the corresponding Markov parameters. We then show that this ``prior knowledge'' on the weak sparsity of the Markov parameters can be systematically captured and exploited via an $\ell_1$-regularized estimation method. Our results imply that the Markov parameters of a partially observed linear system can be learned with certifiable bounds in the high-dimensional settings, where the system dimension is significantly larger than the number of available samples, thereby paving the way towards the efficient learning of massive-scale safety-critical systems. Within the realm of statistics, the $\ell_1$-regularized estimators have been traditionally used to promote (exact) sparsity in the unknown parameters. In this work, we show that a similar $\ell_1$-regularized method can be used to estimate the Markov parameters of the system, \textit{irrespective} of the true sparsity of the unknown system.
\vspace{2mm}

\noindent{\bf Paper organization:} In Section~\ref{sec:related}, we provide a literature review on different system identification techniques, and explain their connection to our work. The problem is formally defined in Section~\ref{sec:problem}, and the main results are presented in Section~\ref{sec:main}. We provide an empirical study of our method on synthetically generated systems in Section~\ref{sec:simulations}, and end with conclusions and future directions in Section~\ref{sec:conclusion}. To streamline the presentation, the proofs are deferred to the appendix.\vspace{2mm}

\noindent{\bf Notation:} Upper- and lower-case letters are used to denote matrices and vectors, respectively. For a matrix $M\in\mathcal{R}^{m\times n}$ the symbols $M_{:j}$ and $M_{j:}$
indicate the $j^{\text{th}}$ column and row of $M$, respectively. Given a vector $v$ and an index set $\mathcal{S}$, the notation $v_{\mathcal{S}}$ refers to a subvector of $v$ whose indices are restricted to the set $\mathcal{S}$. For a vector $v$, $\|v\|_p$ corresponds to its $\ell_p$-norm. For a matrix $M$, the notation $\|M\|_{p,q}$ is equivalent to $\|\begin{bmatrix}
\|M_{1:}\|_p & \|M_{2:}\|_p &\dots&M_{m:}
\end{bmatrix}\|_q$. Moreover, $\|M\|_q$ refers to the induced $q$-norm of the matrix $M$. The notation $\|M\|_F$ is used to denote the Frobenius norm, defined as $\|M\|_{2,2}$. Furthermore, $\rho(M)$ correspond to the spectral radius of $M$. Given the sequences $f(n)$ and $g(n)$ indexed by $n$, the notation $f(n) = \mathcal{O}(g(n))$ or $f(n)\lesssim g(n)$ implies that there exists a universal constant $C<\infty$, independent of $n$, that satisfies $f(n) \leq Cg(n)$. Moreover, $f(n) = \tilde{\mathcal{O}}(g(n))$ is used to denote $f(n) = \mathcal{O}(g(n))$, modulo logarithmic factors. Similarly, the notation $f(n)\asymp g(n)$ implies that there exist constants $C_1>0$ and $C_2<\infty$, independent of $n$, that satisfy $C_1g(n)\leq f(n) \leq C_2g(n)$. Given two scalars $a$ and $b$, the notation $a\vee b$ denotes their maximum. We use $x\sim \mathcal{N}(\mu, \Sigma)$ to show that $x$ is a multivariate random variable drawn from a Gaussian distribution with mean $\mu$ and covariance $\Sigma$.  For two random variables $x$ and $y$, the notation $x\sim y$ implies that they have the same distribution. $\mathbb{E}[x]$ denotes
the expected value of the random variable $x$. For an event $\mathcal{X}$, the notation $\mathbb{P}(\mathcal{X})$ refers to its probability of occurrence. The scalar $c$ denotes a universal constant throughout the paper.

\section{Related Works}\label{sec:related}

{\bf System identification:} Estimating system models from input/output experiments has a well-developed theory dating back to the 1960s, particularly in the case of linear and time-invariant systems. Standard reference textbooks on the topic include \cite{aastrom1971system,ljung1999system,chen2012identification,goodwin1977dynamic}, all focusing on establishing \emph{asymptotic} consistency of the proposed estimators. On the other hand, contemporary results in statistical learning as applied to system identification seek to characterize \emph{finite time and finite data} rates. For fully observed systems,~\cite{dean2017sample} shows that a simple least-squares estimator can correctly recover the system matrices with multiple trajectories whose length scale linearly with the system dimension. This result was later generalized to the single sample trajectory setting for stable~\cite{dean2018regret}, and unstable ~\cite{simchowitz2018learning, dean2019sample, sarkar2018fast} systems, with sample complexities depending polynomially on the system dimension. These results were later extended to stable~\cite{oymak2019non,sarkar2019finite,tsiamis2019finite,simchowitz2019learning}, and unstable~\cite{zheng2020non} partially observed systems, where it is shown that the system matrices (or their associated Markov parameters) can be learned with similar sample complexities.\vspace{2mm}

\noindent{\bf Regularized estimation:} To further reduce the sample complexity of the system identification, a recent line of works has focused on learning dynamical systems with prior information. The works~\cite{fattahi2019learning, fattahi2018data, fattahi2018non, fattahi2018sample} employ $\ell_1$- and $\ell_1/\ell_\infty$-regularized estimators to learn fully observed sparse systems with sample complexities that scale polynomially in the number nonzero entries in different rows and columns of the system matrices, but only logarithmically in the dimension of the system. However, these methods are not applicable to partially observed systems with hidden states. Another line of works~\cite{sun2020finite, wahlberg2013matrix, cai2016robust} introduces a different regularization technique, where the nuclear norm of the Hankel matrix is minimized to improve the sample complexity of learning inherently low-order systems. In particular,~\cite{sun2020finite} shows that for multiple-input-single-output (MISO) systems with order $R\ll n$, the sample complexity of estimating both Markov parameters and Hankel matrix can be reduced to $\mathcal{O}(R^2)$.

\vspace{2mm}

\noindent{\bf Learning-based control:} Complementary to the aforementioned results, a large body of works study adaptive~\cite{dean2018regret, abbasi2011regret, abbasi2019model, lale2020regret}, robust~\cite{dean2019sample, dean2019safely, mania2019certainty}, or distributed~\cite{fattahi2019efficient, furieri2020learning} control of unknown linear systems. These works, culminated under the umbrella of model-based RL, indicate that if a learned model is to be integrated into a safety-critical control loop, then it is essential that the uncertainty associated with the learned model be explicitly quantified. This way, the learned model and the uncertainty bounds can be integrated with a reach body of tools from robust and adaptive control to provide strong end-to-end guarantees on the system performance and stability. 

\section{Problem Statement}\label{sec:problem}
Consider the following linear time-invariant (LTI) dynamical system:
\begin{align}
x_{t+1} &= Ax_t+Bu_t+w_t\\
y_t &= Cx_t+Du_t+v_t
\end{align}
where $x_t\in\mathbb{R}^{n}$, $u_t\in\mathbb{R}^p$, and $y_t\in \mathbb{R}^m$ are the state, input, and output of the system at time $t$. Moreover, the vectors $w_t\in\mathbb{R}^n$ and $v_t\in\mathbb{R}^m$ are the process (or disturbance) and measurement noises, respectively. Throughout the paper, we assume that both $v_t$ and $w_t$ have element-wise independent sub-Gaussian distributions with parameters $\sigma_w$ and $\sigma_v$, respectively. Moreover, without loss of generality, we assume that $x_0 = 0$\footnote{Our results can be readily extended to scenarios where $x_0$ is randomly drawn from a sub-Gaussian distribution.}. The parameters  $A\in\mathbb{R}^{n\times n}$, $B\in \mathbb{R}^{n\times p}$, $C\in\mathbb{R}^{m\times n}$, and $D\in\mathbb{R}^{m\times p}$ are the unknown system matrices, to be estimated from a single input-output sample trajectory $\{(u_t,y_t)\}_{t=0}^{\bar N}$. Much of the progress on the system identification is devoted to learning different variants of \textit{fully observed} systems, where $C = I$ and $v_t = 0$. While being theoretically important, the practicality of these results are limited, since realistic dynamical systems are not directly observable, or corrupted with measurement noise. 

On the other hand, the lack of ``intermediate'' states $x_t$ in partially observed systems gives rise to a mapping from $u_k$ to $y_k$ that is highly nonlinear in terms of the system parameters:
\begin{align}\label{in_out}
y_t = \underbrace{Du_t+\sum_{\tau=1}^{T-1}CA^{\tau-1}Bu_{t-\tau}}_{\text{ Effect of the last $T$ inputs}}+\underbrace{\sum_{\tau=1}^{T-1}CA^{\tau-1}w_{t-\tau}+v_t}_{\text{ Effect of noise}}+\underbrace{\vphantom{\sum_{\tau=1}^{N}} CA^{T-1}Bx_{t-T+1}}_{\text{ Effect of the state at time $t-T+1$}}
\end{align}
\begin{sloppypar}
	\noindent where $t\geq T-1$. The first term in~\eqref{in_out} captures the effect of the past $T$ inputs on $y_t$, while the second term corresponds to the effect of the unknown disturbance and measurement noises on $y_t$. Finally, the third term controls the contribution of the unknown state $x_{t-T+1}$ on $y_t$, whose effect diminishes exponentially fast with $T$, provided that $A$ is stable. A closer look at the first term reveals that the relationship between $y_t$ and $\{u_t,u_{t-1},\dots,u_{t-T+1}\}$ becomes linear in terms of the \textit{Markov matrix} 
	\begin{align}
		G = \begin{bmatrix}
		D\! &\! G_0\! &\! G_1 \dots G_{T-2}
		\end{bmatrix} = \begin{bmatrix}
		D\! &\! CB\! &\! CAB \dots CA^{T-2}B
		\end{bmatrix}\in\mathbb{R}^{m\times Tp},
	\end{align} whose components are commonly known as {Markov parameters} of the system. One of the main goals of this paper is to obtain an accurate estimate of $G$ given a single input-output trajectory. The Markov parameters can be used to directly estimate the outputs of the system from the past input. Moreover, as will be shown later, a good estimation of the Markov parameters can be translated into an accurate estimate of the Hankel matrix, which in turn can be used in the model reduction and $\mathcal{H}_\infty$ methods in control theory~\cite{antoulas2005approximation, zhou1998essentials}. Finally, given the estimated Markov matrix $G$, one can recover estimates of the system matrices. Note that it is only possible to extract the system parameters up to a nonsingular transformation: given any nonsingular matrix $T$, the system matrices $(A,B,C,D)$ and $(T^{-1}AT,TB,CT^{-1},D)$ correspond to the same Markov matrix. Therefore, a common approach for recovering the system matrices is to first construct the associated Hankel matrix, and then extract a \textit{realization} of the system parameters from the Hankel matrix, e.g. via the Ho-Kalman method~\cite{ho1966effective, ljung1999system}. In fact, it has been recently shown in~\cite{oymak2019non, sarkar2019finite} that the Ho-Kalman method can robustly obtain a balanced realization of the system matrices, provided that the estimated Markov matrix enjoys a small estimation error.
\end{sloppypar}

\begin{proposition}[Oymak and Ozay~\cite{oymak2019non}, informal]\label{prop_hokalman}
	Suppose that the true system is controllable and observable. Given an estimate $\widehat{G}$ of $G$, the Ho-Kalman method outputs system matrices $(\widehat{A},\widehat{B},\widehat{C},\widehat{D})$ that satisfy
	\begin{align}
		\|B-\mathcal{U}\widehat{B}\|_F&\lesssim \sqrt{T}\|G-\widehat{G}\|_F\\
		\|C-\widehat{C}\mathcal{U}^{\top}\|_F&\lesssim \sqrt{T}\|G-\widehat{G}\|_F\\
		\|A-\mathcal{U}\widehat{A}\mathcal{U}^{\top}\|_F&\lesssim {T}\|G\|_2\|G-\widehat{G}\|_F
	\end{align}
	for some unitary matrix $\mathcal{U}$, provided that $\widehat{G}$ is sufficiently close to $G$.
\end{proposition}
Therefore, without loss of generality, our focus will be devoted to obtaining accurate estimates of the Markov and Hankel matrices. 
To streamline the presentation, the concatenated input and process noise vectors are defined as:
\begin{align}
\bar{u}_t &= \begin{bmatrix}
u_t^\top & u_{t-1}^\top & \dots & u_{t-T+1}^\top
\end{bmatrix}^\top\in\mathbb{R}^{Tp},\\
\bar{w}_t &= \begin{bmatrix}
w_t^\top & w_{t-1}^\top & \dots & w_{t-T+1}^\top
\end{bmatrix}^\top\in\mathbb{R}^{Tn},
\end{align}
Moreover, the following concatenated matrix will be used throughout the paper:
\begin{align}
F = \begin{bmatrix}
0 & C & CA & \dots & CA^{T-2}
\end{bmatrix}\in \mathbb{R}^{m\times Tn}
\end{align}
Based on the above definitions, the input-output relation~\eqref{in_out} can be written compactly as
\begin{align}
y_t = G\bar{u}_t+F\bar{w}_t+e_t+v_t
\end{align}
where $e_t = CA^{T-1}x_{t-T+1}$.
To estimate the Markov matrix $G$, the work~\cite{oymak2019non} proposes the following least-squares estimator:
\begin{align}
\widetilde{G} = \arg\min_X\sum_{t=T-1}^{N+T-2}\|y_{t}-X\bar{u}_{t}\|_2^2
\end{align}
Define $q = p+n+m$ as the system dimension, and $\sigma^2_e$ as the effective variance of $e_t$, as in \begin{align}
\sigma_e = \Phi(A)\|CA^{T-1}\|\sqrt{\frac{T\|\Gamma_{\infty}\|}{1-\rho(A)^{2T}}}
\end{align}
where 
\begin{align}\label{eq_cov}
\Phi(A) = \sup_{\tau\geq 0} \frac{\|A^\tau\|}{\rho(A)^\tau},\quad \Gamma_{\infty} = \sum_{i=0}^{\infty}\sigma_w^2A^i(A^\top)^i+\sigma_u^2A^iBB^\top(A^\top)^i
\end{align}
The work~\cite{oymak2019non} characterizes the non-asymptotic behavior of the least-squares estimate $\widetilde{G}$.
\begin{theorem}[Oymak and Ozay \cite{oymak2019non}]\label{thm_ls}
	Suppose that $u_t\sim\mathcal{N}(0,\sigma_u^2I)$ for every $t = 0,\dots, T+N-2$, and $N \gtrsim Tq\log^2(Tq)\log^2(Nq)$. Then, with overwhelming probability, the following inequalities hold:
	\begin{align}
	&\|\widetilde{G}-G\|_2 \lesssim \frac{\sigma_v+\sigma_e+\sigma_w\|F\|_2}{\sigma_u}\sqrt{\frac{Tq\log^2(Tq)\log^2(Nq)}{N}},\\
	&\|\widetilde{G}-G\|_F \lesssim \frac{(\sigma_v+\sigma_e)\sqrt{m}+\sigma_w\|F\|_2}{\sigma_u}\sqrt{\frac{Tq\log^2(Tq)\log^2(Nq)}{N}}\label{bound_ls_F}
	\end{align}
	where $\sigma^2_u$, $\sigma^2_v$, $\sigma^2_w$ are the variances of the random input, disturbance noise, and the measurement noise, respectively.
\end{theorem}
The above theorem shows that the spectral norm of the estimation error for the Markov parameters via least-squares method is in the order of $\tilde{\mathcal{O}}\left(\sqrt{{T(n+m+p)}/{N}}\right)$, provided that $N = \tilde{\mathcal{O}}(T(n+m+p))$. Moreover,~\cite{oymak2019non} shows that the number of samples $N$ can be reduced to $\tilde{\mathcal{O}}(Tp)$ (without improving the spectral norm error). Such dependency on the system dimension is unavoidable if one does not exploit any prior information on the structure of $G$: roughly speaking, the Markov parameter $G$ has $Tmp$ unknown parameters, and one needs to collect at least $Tp$ outputs (each with size $m$) to obtain a well-defined least-squares estimator. 
Evidently, such dependency on the system dimension may be prohibitive for large-scale and safety-critical systems, where it is expensive to collect as many output samples. Motivated by this shortcoming of the existing methods, we aim to address the following open question: \vspace{2mm}

\noindent {\bf Question: }{\it Can partially observed linear systems be learned in a logarithmic sample complexity?}

\section{Main Results}\label{sec:main}
In this section, we provide an affirmative answer to the aforementioned question. At a high-level, we will use the fact that, due to the stability of $A$, the Markov parameters decay \textit{exponentially fast}, which in turn implies that the rows of the extended matrix $G$ exhibit a bounded $\ell_1$-norm (also known as weak sparsity~\cite{wainwright2019high}). This observation strongly motivates the use of the following regularized estimator:
\begin{align}\label{lasso1}
\widehat{G} = \arg\min_X\left(\frac{1}{2N}\sum_{t=T-1}^{N+T-2}\|y_{t}-X\bar{u}_{t}\|_2^2\right)+\lambda\|X\|_{1,1}
\end{align}
\begin{sloppypar}
	\noindent Due to the stability of $A$, there exist scalars $\Csys\geq 1$ and $\rho<1$ such that $\|A^\tau\|_1\leq \Csys\rho^\tau$. Without loss of generality and to simplify the notation, we assume that $\max\{\|B\|_1, \|C\|_1, \|D\|_1\}\leq \Csys$. Finally, define the \textit{effective variance} of the disturbance noise as $\bar{\sigma}_w = \left(\frac{\Csys^2}{1-\rho}\right)\sigma_w$.
	The main result of the paper is the following theorem:
\end{sloppypar}
\begin{theorem}\label{thm_main}
	Suppose that $u_t\sim\mathcal{N}(0,\sigma_u^2I)$ for every $t = 0,\dots, T+N-2$. Moreover, suppose that $N$ and $T$ satisfy the following inequalities:
	\begin{align}\label{TN}
	N\gtrsim \log^2(Tp),\qquad T\gtrsim T_0 = \frac{\log\log(Nn+Tp)+\log\left(\frac{\Csys}{1-\rho}\right)+\log(\sigma_w+\sigma_v)+\log\left(\frac{1}{\epsilon}\right)}{1-\rho}
	\end{align}
	for an arbitrary $\epsilon>0$. Finally, assume that $\lambda$ is chosen such that
	\begin{align}\label{lambda}
	\lambda \asymp \sigma_u\left(\bar\sigma_w+\sigma_v\right)\sqrt{\frac{\log(Tpn)}{N}}+\epsilon
	\end{align}
	Then, with overwhelming probability, the following inequalities hold:
\begin{align}
\|G-\widehat{G}\|_{2,\infty}&\lesssim \mathcal{E}_1\vee\mathcal{E}_2\label{bound_2inf}\\
\|G-\widehat{G}\|_F&\lesssim \sqrt{m}(\mathcal{E}_1\vee\mathcal{E}_2)\label{bound_F}
\end{align}
where 
\begin{align}
	&\mathcal{E}_1 = \sqrt{\frac{\Csys^3}{1-\rho}}\left(\sqrt{\frac{\bar{\sigma}_w+\sigma_v}{\sigma_u^3}}\left(\frac{\log(Tpn)}{N}\right)^{1/4}+\frac{\epsilon}{\sigma_u^2}\right)\nonumber\\
	&\mathcal{E}_2 = {\frac{\Csys^3}{1-\rho}}\left(\frac{\log(Tp)}{N}\right)^{1/4}
\end{align}
\end{theorem}
The above theorem can be used to provide estimation error bounds on the higher order Markov parameters and Hankel matrices (which can be used to recover a realization of the system parameters $\{A,B,C,D\}$, as delineated in Proposition~\ref{prop_hokalman}). Similar to~\cite{oymak2019non}, define the true and estimated $K^{\text{th}}$ order (where $K\geq T$) Markov parameters as
\begin{align}
	G^{(K)} &= \begin{bmatrix}
	D & CB & CAB & \dots & CA^{K-2}B
	\end{bmatrix}\in\mathbb{R}^{m\times Kp},\\ \widehat{G}^{(K)} &= \begin{bmatrix}
	\widehat{G} & \mathbf{0}_{m\times (K-T)p}
	\end{bmatrix}\in\mathbb{R}^{m\times Kp}
\end{align}
Moreover, define the true and estimated $K^{\text{th}}$ order Hankel matrices as
\begin{align}
	&H^{(K)} = \begin{bmatrix}
	D & CB & \dots & CA^{K-2}B\\
	CB & CAB & \dots & CA^{K-1}B\\
	 & & \vdots & \\
	 CA^{K-2}B & CA^{K-1}B & \dots & CA^{2K-3}B
	\end{bmatrix}\in\mathbb{R}^{Km\times Kp},\\
	& \widehat{H}^{(K)} = \begin{bmatrix}
	\widehat D & \widehat G_0 & \dots & \widehat G_{T-3} &\widehat G_{T-2} & \mathbf{0}_{m\times n} & \dots & \mathbf{0}_{m\times n}\\
	\widehat G_0 & \widehat G_1 & \dots & \widehat G_{T-2} &\mathbf{0}_{m\times n} & \mathbf{0}_{m\times n} & \dots & \mathbf{0}_{m\times n}\\
	& & \vdots & \\
	\widehat G_{T-2} & \mathbf{0}_{m\times n} & \dots & \mathbf{0}_{m\times n} &\mathbf{0}_{m\times n} & \mathbf{0}_{m\times n} & \dots & \mathbf{0}_{m\times n}\\
	\mathbf{0}_{m\times n} & \mathbf{0}_{m\times n} & \dots & \mathbf{0}_{m\times n} &\mathbf{0}_{m\times n} & \mathbf{0}_{m\times n} & \dots & \mathbf{0}_{m\times n}\\
	& & \vdots & \\
	\mathbf{0}_{m\times n} & \mathbf{0}_{m\times n} & \dots & \mathbf{0}_{m\times n} &\mathbf{0}_{m\times n} & \mathbf{0}_{m\times n} & \dots & \mathbf{0}_{m\times n}
	\end{bmatrix}\in\mathbb{R}^{Km\times Kp}
\end{align}
 Our next corollary follows from Theorem~\ref{thm_main}.
 \begin{corollary}\label{cor_HM}
 	Suppose that $u_t\sim\mathcal{N}(0,\sigma_u^2I)$ for every $t = 0,\dots, T+N-2$, and $N$ and $\lambda$ satisfy~\eqref{TN} and~\eqref{lambda}, respectively. Moreover, assume that $T\gtrsim T_0\vee\left(\log(\|C\|_\infty)+\log{(1/\tilde\epsilon)}\right)/(1-\rho)$ for an arbitrary $\tilde{\epsilon}>0$. Then, for any $K\geq T$ (including $K=\infty$), the following inequalities hold with overwhelming probability:
 	\begin{align}
 		\|G^{(K)}-\widehat{G}^{(K)}\|_{2,\infty} &\lesssim \mathcal{E}_1\vee\mathcal{E}_2+\tilde{\epsilon},\label{GK_2inf}\\
 		\|G^{(K)}-\widehat{G}^{(K)}\|_{F} &\lesssim \sqrt{m}(\mathcal{E}_1\vee\mathcal{E}_2+\tilde{\epsilon})\label{GK_F}\\
 		\|H^{(K)}-\widehat{H}^{(K)}\|_{2,\infty} &\lesssim \mathcal{E}_1\vee\mathcal{E}_2+\tilde{\epsilon},\label{HK_2inf}\\
 		\|H^{(K)}-\widehat{H}^{(K)}\|_{F} &\lesssim \sqrt{Tm}(\mathcal{E}_1\vee\mathcal{E}_2+\tilde{\epsilon})\label{HK_F}
 	\end{align}
 \end{corollary}
Next, we will explain the implications of Theorem~\ref{thm_main} and Corollary~\ref{cor_HM}.\vspace{2mm}
	
	\noindent \textbf{Sample complexity:} According to Theorem~\ref{thm_main} and Corollary~\ref{cor_HM}, the required number of samples $N$ for estimating the Markov parameters and the Hankel matrix scales poly-logarithmically with the system dimension, making it particularly well-suited to massive-scale dynamical systems, where the system dimension surpasses the number of available input-output samples. In contrast, the existing methods for learning partially observed linear systems do not provide any guarantee on their estimation errors under such ``high-dimension/low-sampling'' regime. Moreover, the imposed lower bound on $T$ scales double-logarithmically with respect to the system dimension\footnote{The imposed lower bound on $T$ is to simplify the derived bounds, and hence, can be relaxed at the expense of less intuitive estimation bounds.}, which can be treated as a constant number for all practical purposes.\footnote{It is easy to verify that $\log\log(s)\leq 5$ for any $s\leq 10^{50}$!}\vspace{2mm}
	
	\noindent{\bf Estimation error:} The estimation error bounds in Theorem~\ref{thm_main} and Corollary~\ref{cor_HM} are in terms of the row-wise $\ell_2$ and Frobenius norms. In contrast, most of the existing methods provide upper bounds on the spectral norm of the estimation error. An important benefit of the provided row-wise bound is that it provides a finer control over the element-wise estimation error, which in turn can be used in the recovery of the special sparsity patterns in the Hankel matrices~\cite{jin2017sparse}. We note that although the provided bound on the Frobenius norm of the estimation error readily applies to its spectral norm, we believe that it can be strengthened. Moreover, the provided estimation error bound reduces at the rate $N^{-1/4}$, which is slower than the rate $N^{-1/2}$ for the simple least-squares estimator (see Theorem~\ref{thm_ls}). However, a more careful scrutiny of~\eqref{bound_F} and~\eqref{bound_ls_F} reveals that our proposed estimator outperforms the least-squares in the regime where
	\begin{align}
		\frac{N}{T^2}\lesssim \frac{q^2\log^4(Tq)\log^4(Nq)}{\log(Tpn)} = \tilde{O}\left((n+m+p)^2\right)
	\end{align}
	In fact, a stronger statement can be made on the ratio between the Frobenius norms of the estimation errors:
	\begin{corollary}\label{cor_bound}
		Denote the right hand sides of~\eqref{bound_ls_F} and~\eqref{bound_F} as $\mathcal{E}^{LS}_{F}$ and $\mathcal{E}^{\ell_1}_F$, respectively. Suppose that $\sigma_u\vee\sigma_w\vee\sigma_v\vee\frac{\Csys}{1-\rho}\vee \Phi(A)\vee \|B\|_2\vee \|C\|_2 = \mathcal{O}(1)$, and $T\gtrsim T_0+\log(n+m+p)$. Then, we have
		\begin{align}
			\lim_{n,m,p\to\infty}\frac{\mathcal{E}^{\ell_1}_F}{\mathcal{E}^{LS}_F} = 0
		\end{align} 
		provided that $T$ and $N$ satisfy
		\begin{align}\label{eq_assume}
			\lim_{n,m,p\to\infty}\frac{N\log(Tpn)}{T^2(n+m+p)^2} = 0
		\end{align}
	\end{corollary}

\begin{table}[h]
	\centering
	{\begin{tabular}{|c|c|c|c|}
		\hline
		Method & Sample Complexity & Error Bound ($\|\cdot\|_F$) & Additional Notes\\
		\hline
		proposed method & $\mathcal{O}(\log^2(Tp))$& $\mathcal{O}\left(\sqrt{m}\left(\frac{\log(Tnp)}{N}\right)^{1/4}\right)$ & \footnotesize\begin{tabular}{@{}c@{}} Single trajectory\end{tabular}\\
		\hline
		Oymak and Ozay~\cite{oymak2019non} & $\tilde{\mathcal{O}}(Tq)$ & $\tilde{\mathcal{O}}\left(\sqrt{m}\left(\frac{Tq}{N}\right)^{1/2}\right)$ & \footnotesize\begin{tabular}{@{}c@{}} Single trajectory\end{tabular}\\
		\hline
		Sarkar {\it et. al.}~\cite{sarkar2019finite} & $\tilde{\mathcal{O}}(n^2)$ & $\tilde{\mathcal{O}}\left(\sqrt{m}\left(\frac{pn^2}{N}\right)^{1/2}\right)$& \footnotesize\begin{tabular}{@{}c@{}} Single trajectory,\\ Suitable for systems with\\ unknown order\\ \end{tabular}\\
		\hline
		Zheng and Li~\cite{zheng2020non} & $\tilde{O}(mT+q)$ & $\tilde{\mathcal{O}}\left(\sqrt{m}\left(\frac{T^{3}q}{N}\right)^{1/2}\right)$ & \footnotesize\begin{tabular}{@{}c@{}} Multiple trajectories,\\ Stable and unstable systems\\ \end{tabular}\\
		\hline
		Sun {\it et. al.}~\cite{sun2020finite} & $\tilde{\mathcal{O}}(pR)$ & $\tilde{\mathcal{O}}\left(\left(\frac{Rnp}{N}\right)^{1/2}\right)$& \footnotesize\begin{tabular}{@{}c@{}} Multiple trajectories,\\ MISO ($m=1$)\end{tabular}\\
		\hline
		Tu {\it et. al.}~\cite{tu2017non} & $\tilde{\mathcal{O}}(r)$ & $\tilde{\mathcal{O}}\left(\left(\frac{r}{T}\right)^{1/2}\right)$& \footnotesize\begin{tabular}{@{}c@{}} Multiple trajectories,\\ SISO ($p = m=1$)\end{tabular}\\
		\hline
	\end{tabular}}
	\caption{\footnotesize Sample complexity and error bounds on the estimated Markov parameters for different methods. The parameters $R\leq n$ and $r$ are respectively the order of the system and the length of the FIR impulse response; see~\cite{sun2020finite} and~\cite{tu2017non} for more information. The error bounds are measured with respect to the Frobenius norm.
	}
	\label{table_comparison}
\end{table}

The above proposition implies that in the regime where $N$ is not significantly larger than the system dimension, the derived upper bound on the estimation error of the regularized estimator becomes arbitrarily smaller than that of the simple least-squares method. Our numerical analysis in Section~\ref{sec:simulations} also reveals the superior performance of the proposed estimator, even when $N\gg Tp$. Finally, we point out that similar error bounds have been derived for linear regression problems with weakly sparse structures. In particular, the work~\cite{negahban2012unified} considers a ``simpler'' linear model where the samples/outputs are assumed to be independent, and shows that a $\ell_1$-regularized estimator achieves an error bound in the order of $\mathcal{O}\left((\log(d)/N)^{1/4}\right)$, where $d$ is the dimension of the unknown regression vector. Theorem~\ref{thm_main} reveals that the same non-asymptotic rates can be achieved in the context of system identification with a single (and correlated) input-output trajectory. Table~\ref{table_comparison} compares the performance of the proposed estimator with other state-of-the-art methods.\vspace{2mm}

\noindent{\bf Role of signal-to-noise ratio:} Intuitively, the estimation error should improve with an increasing signal-to-noise (SNR) ratio (in our problem, the SNR ratio is defined as $\sigma_u^3/(\sigma_w+\sigma_v)$); this behavior is also observed in the related works~\cite{oymak2019non, sun2020finite, tu2017non}. In contrast, our provided bound is the maximum of two terms, one of which is independent of the SNR ratio. In other words, an increasing SNR ratio can only shrink the estimation error down to a certain positive threshold. The reason behind this seemingly unintuitive behavior lies in the statistical behavior of the random input matrix $U$. For two different vectors $\zeta$ and $\tilde\zeta$, the quantity $U(\zeta-\tilde{\zeta})$ measures how \textit{distinguishable} these vectors are under the considered linear model. For the cases where $N\gtrsim Tp$, it is easy to see that these two vectors are easily distinguishable, since $\|U(\zeta-\tilde{\zeta})\|_2^2\geq \kappa\sigma_u^2 \|\zeta-\tilde{\zeta}\|_2^2$ holds with high probability, for some strictly positive $\kappa$ (see, e.g.,~\cite{oymak2019non, sarkar2019finite}). However, in the high-dimensional setting, where $N\ll Tp$, the matrix $U$ will inevitably have zero singular values, and hence, $\|U(\zeta-\tilde{\zeta})\|_2^2\geq \kappa\sigma_u^2 \|\zeta-\tilde{\zeta}\|_2^2$ does not hold for specific choices of $\zeta$ and $\tilde{\zeta}$. Under such circumstances, we will show that the \textit{relaxed} inequality $\|U(\zeta-\tilde{\zeta})\|_2^2\geq \kappa\sigma_u^2 \|\zeta-\tilde{\zeta}\|_2^2-\sigma_u^2f(\zeta-\tilde{\zeta})$ holds for any $\zeta$ and $\tilde{\zeta}$, where $f(\cdot)$ is a function to be defined later. Upon replacing $\zeta-\tilde{\zeta}$ with $G_{i:}-\widehat{G}_{i:}$ for an arbitrary row index $i$, it is easy to see that the derived lower bound becomes nontrivial {\it only if} $\|G_{i:}-\widehat{G}_{i:}\|_2^2> f(G_{i:}-\widehat{G}_{i:})/\kappa$, which is \textit{independent} of the SNR ratio.
We will formalize this intuition later in the proof of Theorem~\ref{thm_main}. In particular, we will show that: (1) the threshold $f(G_{i:}-\widehat{G}_{i:})/\kappa$ is small, i.e., it is upper bounded by $\mathcal{E}_2^2$; (2) whenever $\|G_{i:}-\widehat{G}_{i:}\|_2^2$ is larger than $\mathcal{E}_2^2$, it can be upper bounded by $\mathcal{E}_1^2$. 

\begin{figure}[h]
	\centering
	\includegraphics[scale=0.58]{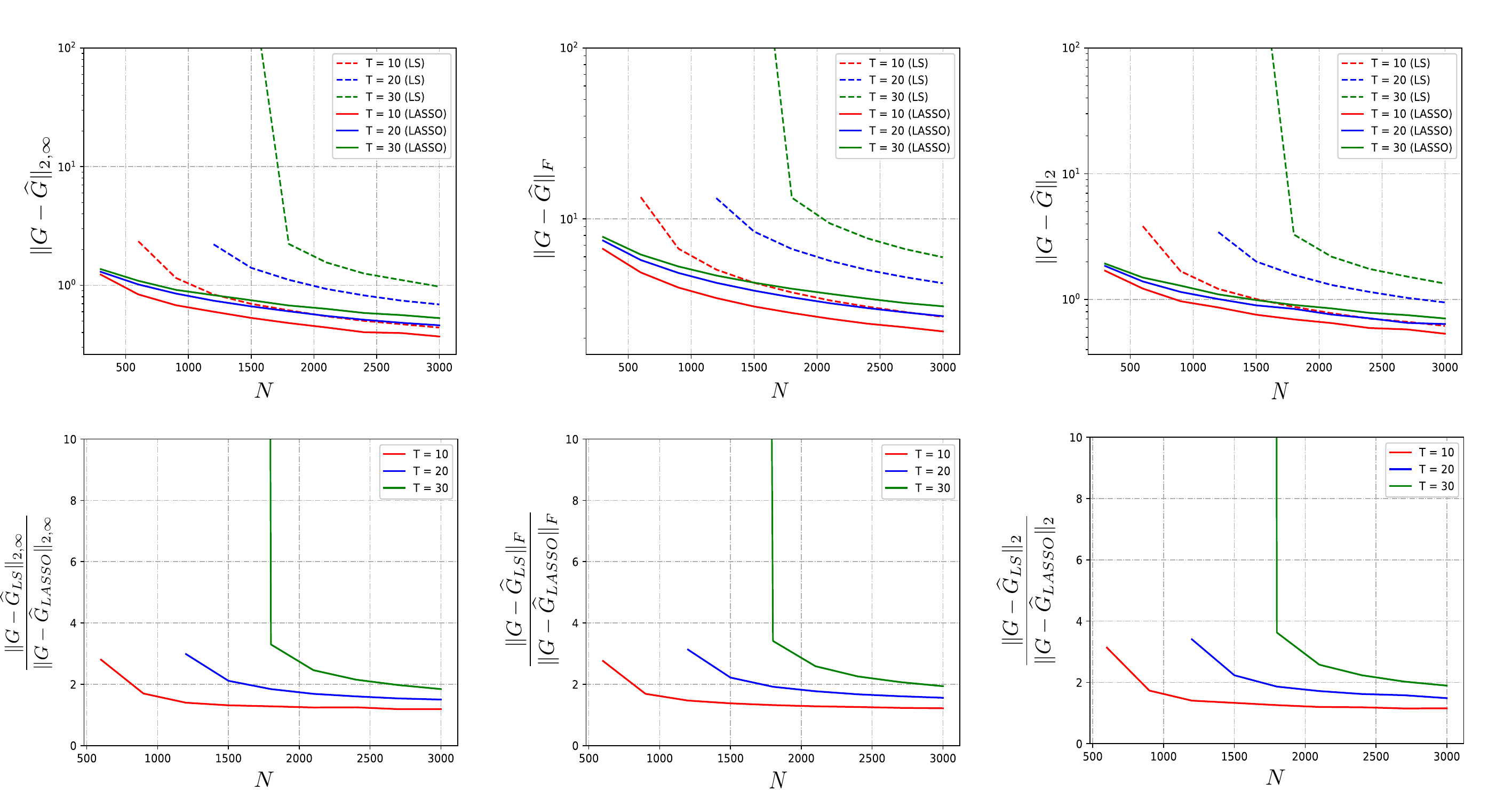}
	\caption{The estimation error of the Markov parameters for \texttt{LASSO} (denoted as  $\widehat{G}_{LASSO}$) and \texttt{LS} (denoted as $\widehat{G}_{LS}$) with respect to the sample size, with $\sigma_w^2 = \sigma_v^2 = 0.1$ and varying $T$. When $N<Tp$, \texttt{LASSO} achieves small estimation error, while \texttt{LS} is not well-defined. Moreover, \texttt{LASSO} significantly outperforms \texttt{LS} when $N\geq Tp$. The $y$-axis in all figures are clipped to better illustrate the differences in the curves.}
	\label{fig:markov}
\end{figure}
\section{Simulations}\label{sec:simulations} In this section, we showcase the performance of the proposed regularized estimator. In particular, we will provide an empirical comparison between our method and the least-squares approach of Oymak and Ozay~\cite{oymak2019non}.\footnote{It has been recently verified in~\cite{sun2020finite} that the method proposed by Oymak and Ozay~\cite{oymak2019non} outperforms that of Sarkar \textit{et. al.}~\cite{sarkar2019finite}. Therefore, without loss of generality, we will focus on the former.} In all of our simulations, we set $n = 200$, $m = p = 50$, and $D = 0$. The system matrices are generated according to the following rules:

\begin{itemize}
	\item[-] $A$ is chosen as a banded matrix, with the bandwidth equal to 5. This implies that the rows and columns of $A$ have at least 6 and at most 11 elements. Moreover, each nonzero entry of $A$ is selected uniformly from $[-0.5,0.5]$. To ensure the stability of the system, $A$ is further normalized to ensure that $\rho(A) = 0.8$. The special structure of $A$ entails that $\|A\|_1$ remains small.
	\item[-] The $(i,j)^{\text{th}}$ entry of $B$ is set to 1 if $i = 4j$, and it is set to 0 otherwise, for every $(i,j)\in \{1,\dots, n\}\times\{1,\dots, p\}$.
	\item[-] $C$ is chosen as a Gaussian matrix, with entries selected from  $\mathcal{N}(0,1/m)$.
\end{itemize}
Note that, despite the sparse nature of $A$ and $B$, the Markov parameters of the system are fully dense, due to the dense nature of $C$. 
Throughout our simulations, $\sigma_u$ is set of $1$, and the values of $\sigma_w$ and $\sigma_v$ are changed to examine the effect of SNR ratio on the quality of our estimates. Moreover, in all of our simulations, we set the regularization parameter to 
\begin{align}\label{key}
	\lambda = 0.2(\sigma_w+\sigma_v)\sqrt{\frac{\log(Tpn)}{N}}+0.02\times 0.8^T
\end{align}
Note that the above choice of the regularization parameter does not require any further fine-tuning, and it is in line with Theorem~\ref{thm_main}, after replacing $\epsilon$ with $0.02\times 0.8^T$ in~\eqref{lambda}. The exponential decay in $\epsilon$ correctly captures the diminishing effect of the unknown initial state $x_{t-T+1}$ on the output $y_t$ with $T$ (see equation~\eqref{in_out}). We point out that a better choice of $\lambda$ may be possible via cross-validation~\cite{shao1993linear}. Figure~\ref{fig:markov} shows the estimation error of the proposed method compared to the least-squares estimator (referred to as \texttt{LASSO} and \texttt{LS}, respectively) for  $\sigma_w^2=\sigma_v^2=0.1$ (averaged over 10 independent trials). It can be seen that \texttt{LASSO} significantly outperforms \texttt{LS} for all values of $N$ and $T$. In the high-dimensional setting, where $N<Tp$, \texttt{LS} is not well-defined, while \texttt{LASSO} results in small estimation errors. Moreover, when $N\geq Tp$, the incurred estimation error of \texttt{LASSO} is $1.2$ to $1077$ times smaller than that of \texttt{LS}. Although the main strength of \texttt{LASSO} is in the high-dimensional regime, it still outperforms \texttt{LS} when $N\gg Tp$. Furthermore, Figure~\ref{fig:hankel} shows the superior performance of \texttt{LASSO} compared to \texttt{LS} in the estimated Hankel matrices.

\begin{figure}[h]
	\centering
	\includegraphics[scale=0.6]{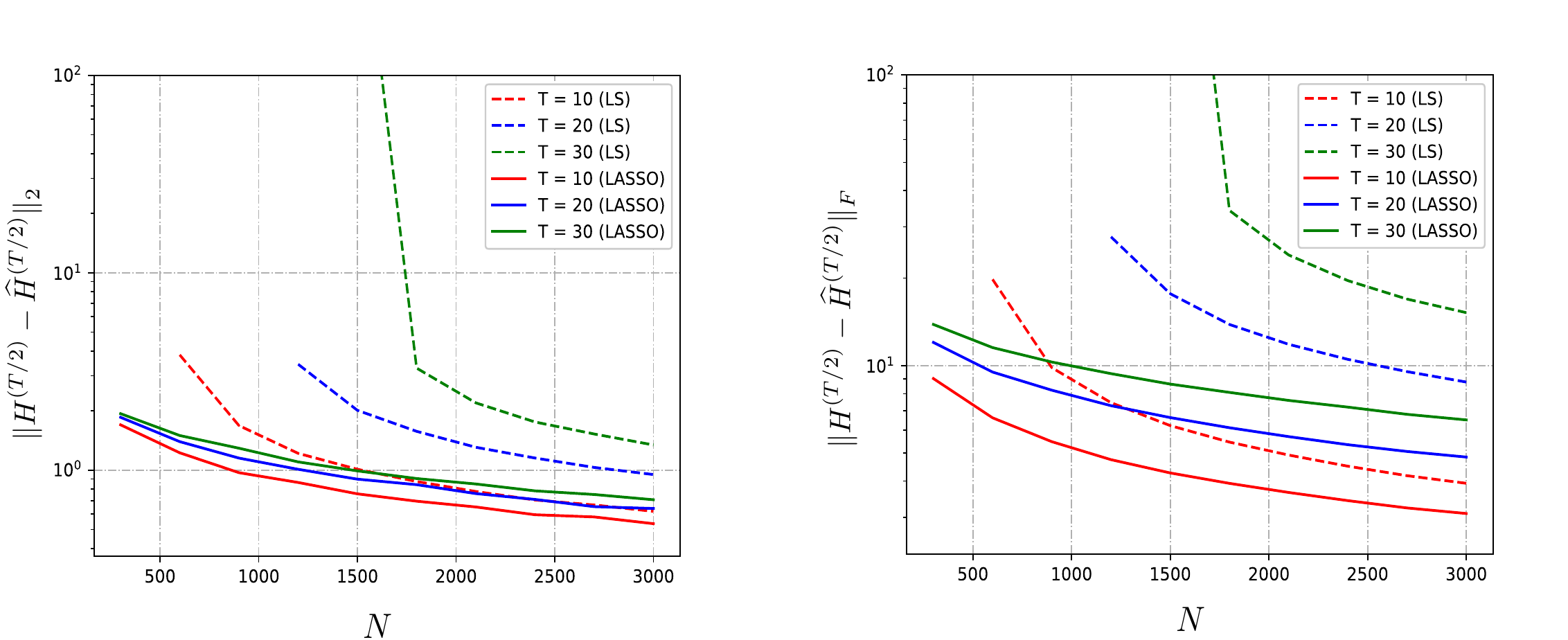}
	\caption{The estimation error of the Hankel matrices for \texttt{LASSO} and \texttt{LS} with respect to the sample size $N$, with $\sigma_w^2 = \sigma_v^2 = 0.1$ and varying $T$. Similar to Figure~\ref{fig:markov}, \texttt{LASSO} outperforms \texttt{LS} for all values of $N$ and $T$. The $y$-axis in all figures are clipped to better illustrate the differences in the curves.}
	\label{fig:hankel}
\end{figure}

Next, we fix $N = 200$, reduce the variance of the disturbance and measurement noises to $\sigma_w^2=\sigma_v^2=0.02$, and report the estimation error for different values of $T$ in Figure~\ref{fig:noise} (left).
To explain the non-monotonic behavior of $\|G-\widehat G\|_F$, first note that incurred estimation error stems from two sources: $(i)$ the measurement and disturbance noises; and $(ii)$ the unknown initial state. For small values of $T$, the number of unknown parameters in $G$ is small, and it can be well-estimated with sufficiently large $N$. However, the effect of the unknown initial state is significant due to the small ``mixing time'', thereby giving rise to a large estimation error. As $T$ grows, the effect of the unknown initial state diminishes exponentially fast, while the size of $G$ (and the number of unknown parameters) increases. Therefore, the estimation error has a non-monotonic dependency on $T$ for any fixed $N$; such behavior is also reflected in Theorem~\ref{thm_main}, after setting $\epsilon = \rho^T$. 
Finally, Figure~\ref{fig:noise} (right) depicts the estimation error of \texttt{LASSO} and \texttt{LS} for different noise variances. It can be seen that the estimation accuracy of \texttt{LASSO} is less sensitive to noise, i.e., it deteriorates at a slower rate with the increasing noise levels.

\begin{figure}[h]
	\centering
	\includegraphics[scale=0.6]{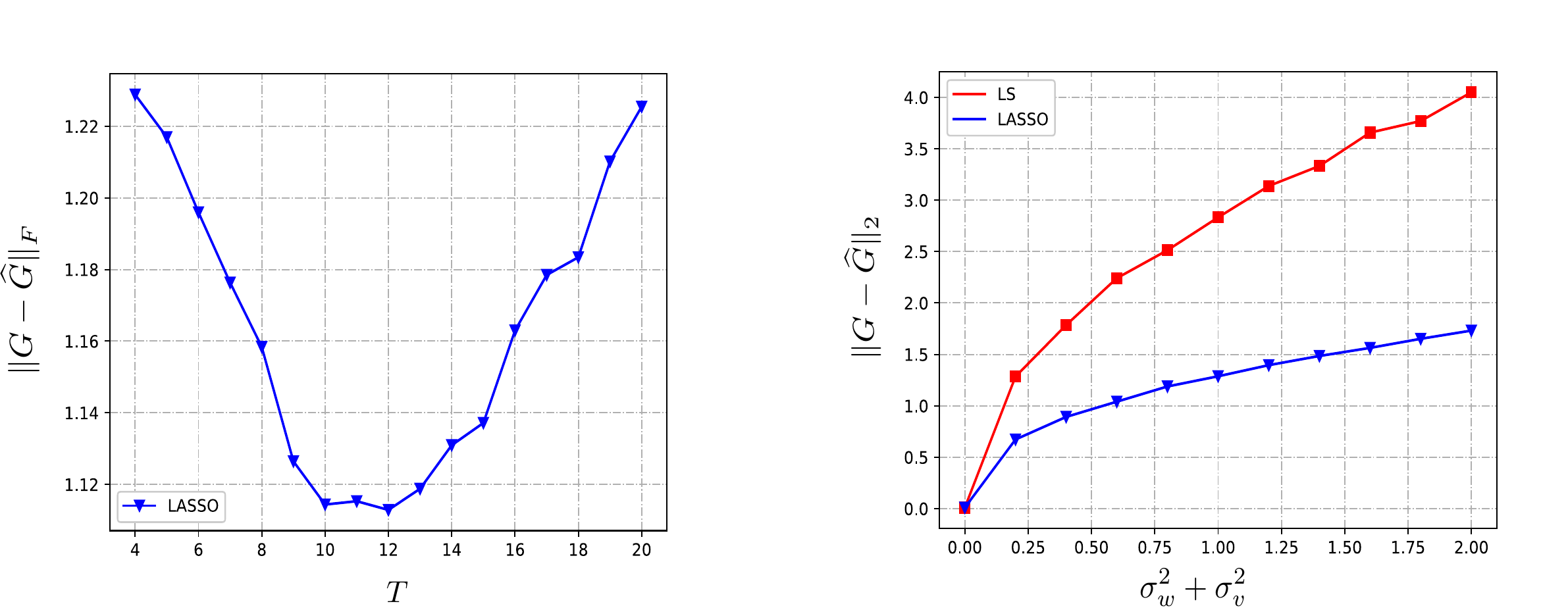}
	\caption{(Left) The estimation error of the Markov parameters for \texttt{LASSO} with respect to $T$, with $N = 2000$, and $\sigma_v^2 = \sigma_u^2 = 0.02$. The estimation error has a non-monotonic behavior with respect to $T$. (Right) The estimation error of the Markov parameters for different noise levels, with $N = 2000$, and $T = 20$. It can be seen that \texttt{LASSO} is less sensitive to the increasing noise levels.} 
	\label{fig:noise}
\end{figure}

\section{Conclusions and Future Directions}\label{sec:conclusion}

In this paper, we propose a method for learning partially observed linear systems from a single sample trajectory in high-dimensional settings, i.e., when the number of samples is less than the system dimension. Most of the existing inference methods presume and rely on the availability of prohibitively large number of samples collected from the unknown system. In this work, we address this issue by reducing the sample complexity of estimating the Markov parameters of partially observed systems via an $\ell_1$-regularized estimator. We show that, when the system is inherently stable, the required number of samples for a reliable estimation of the Markov parameters scales poly-logarithmically with the dimension of the system.

As a promising direction for future research, we will study the sparse recovery of the system matrices from the estimated Markov parameters. Indeed, most of the real-world systems consist of many subsystems with \textit{local} interactions, thereby giving rise to sparse system matrices. However, it is easy to see that sparsity in the system parameters does translate into sparsity in the Markov parameters. On the other hand, the classical system identification methods, such as Ho-Kalman method, often extracts a \textit{dense realization} of the system parameters, and therefore, cannot incorporate prior information, such as sparsity. As a future direction, we aim to remedy this challenge in a principled manner. Given a sparse realization of the system matrices, our next goal is to design a robust distributed controller for the true system, taking into account the uncertainty in the estimated model.

\section*{Acknowledgments}
We would like to thank Necmiye Ozay, Nikolai Matni, Yang Zheng, and Kamyar Azizzadenesheli for their insightful suggestions and constructive comments.
\appendix
\section*{Appendix}
\section{Proof of the Main Results}
In this section, we present the proofs of Theorem~\ref{thm_main} and Corollaries~\ref{cor_HM} and~\ref{cor_bound}. For simplicity, define the error matrix $\Delta = G-\widehat{G}$, and the following concatenated matrices:
\begin{align}
Y &= \begin{bmatrix}
y_{T-1} & y_{T} & \dots & y_{T+N-2}
\end{bmatrix}^\top\in \mathbb{R}^{N\times m}\\
U &= \begin{bmatrix}
\bar{u}_{T-1} & \bar{u}_{T} & \dots & \bar{u}_{T+N-2}
\end{bmatrix}^\top\in\mathbb{R}^{N\times Tp}\\
W &= \begin{bmatrix}
\bar{w}_{T-1} & \bar{w}_{T} & \dots & \bar{w}_{T+N-2}
\end{bmatrix}^\top\in\mathbb{R}^{N\times Tn}
\end{align}
\begin{align}
E &= \begin{bmatrix}
e_{T-1} & e_{T} & \dots & e_{T+N-2}
\end{bmatrix}^\top\in \mathbb{R}^{N\times m}\\
V &= \begin{bmatrix}
v_{T-1} & v_{T} & \dots & v_{T+N-2}
\end{bmatrix}^\top\in \mathbb{R}^{N\times m}
\end{align}
With these definitions, one can re-write~\eqref{in_out} as
\begin{align}
Y = UG^\top + WF^\top + E + V
\end{align}
Moreover, the the $\ell_1$-regularized estimator~\eqref{lasso1} reduces to
\begin{align}\label{lasso2}
\widehat{G} = \arg\min_X \frac{1}{2N}\left\|Y-UX^\top\right\|_F^2+\lambda\|X\|_{1,1}
\end{align}
Note that~\eqref{lasso2} is decomposable over different rows of $X$. Therefore, one can write:
\begin{align}\label{lasso_i}
\widehat{G}_{i:} = \arg\min_X \frac{1}{2N}\left\|Y_{:i}-U(X_{i:})^\top\right\|_F^2+\lambda\|X_{i:}\|_{1,1},\qquad \text{for every } i = 1,2,\dots, m
\end{align}
At the core of our result is the following fundamental lemma, which deterministically bounds the row-wise error of $\widehat{G}$.
\begin{proposition}[Deterministic Guarantee]\label{prop_deterministic}
	Fix a row index $i$, and assume that the following conditions hold:
	\begin{itemize}
		\item[1.] ($\ell_1$-boundedness) We have $\|G_{i:}\|_1\leq R$, for some $R>0$.
		\item[2.] (Restricted singular value) There exists a function $f(\cdot)$ such that 
		\begin{align}\label{svd_restricted}
		\frac{1}{N}\|U\Delta_{i:}\|_{2}^2\geq \kappa\|\Delta_{i:}\|_2^2-f(\Delta_{i:})
		\end{align}
		for some $\kappa\leq 1$.
		\item[3.] (Bound on $\lambda$) We have
		\begin{align}
		\lambda\geq \frac{2}{N}\left(\|U^\top WF_{i:}^\top\|_\infty+\|U^\top E_{:i}\|_\infty+\|U^\top V_{:i}\|_\infty\right)
		\end{align}
	\end{itemize}
	Then, the following inequality holds:
	\begin{align}\label{bound_det}
	\|\Delta_{i:}\|_2^2\leq\max\left\{\frac{2}{\kappa}f(\Delta_{i:}),\frac{88R}{\kappa^2}\lambda\right\}
	\end{align}
\end{proposition}
\begin{proof}
	See Appendix~\ref{app_prop1}.
\end{proof}

Before presenting the implications of this proposition, let us briefly explain the intuition behind the imposed assumptions. It can be easily seen that the first assumption holds for a choice of $R$ that only depends on $\Csys$ and $\rho$ (see Proposition~\ref{prop_G}). On the other hand, the second assumption implies that the concatenated input matrix $U$ has a nonzero singular value in the subspace spanned by the vector $\Delta_{i:}$, which is offset by a ``slack'' term $f(\Delta_{i:})$. We consider two scenarios to explain the inclusion of the slack term. First, in the low-dimensional regime, where $Tp\gtrsim N$ (modulo logarithmic factors), the standard concentration bounds on the random circulant matrices~\cite{krahmer2014suprema, sarkar2019finite} entail that~\eqref{svd_restricted} holds for some uniform constant $\kappa>0$, and with the choice of $f(\Delta_{i:}) = 0$. However, in the high-dimensional settings where $Tp< N$, one has to choose nonzero values for $f(\Delta_{i:})$, since the matrix $U$ will have zero singular values. While the naive choice of $f(\Delta_{i:}) = \kappa \|\Delta_{i:}\|_2^2$ is always feasible, one of the key contributions of this paper is to provide a sharper choice for $f(\Delta_{i:})$ that is particularly well-suited in the context of high-dimensional system identification. In particular, we will show that, for $\Delta_{i:}$ with bounded $\ell_1$-norm, the inequality~\eqref{svd_restricted} holds with high probability, for the choices $\kappa = \sigma^2_u/4$ and $f(\Delta_{i:})\lesssim\sigma^2_uR^2\sqrt{{\log(Tp)}/{N}}$ (see Proposition~\ref{prop_f}). Finally, the third assumption provides a lower bound on the regularization coefficient, which will be shown to hold with overwhelming probability when $\lambda$ is chosen as $(\sigma_u\bar{\sigma}_w+\sigma_u\sigma_v)\sqrt{\log(Tpn)/N}$ (see Proposition~\ref{prop_lambda}).

\begin{proposition}[$\ell_1$-boundedness]\label{prop_G}
	The following inequality holds for every $i=1,\dots,m$:
	\begin{align}
	\|G_{i:}\|_1\leq \frac{2\Csys^3}{1-\rho}
	\end{align}
\end{proposition}
\begin{proof}
	One can write
	\begin{align}
	\|G_{i:}\|_1 = \|D\|_1+\sum_{\tau=0}^{T-2} \|G_\tau\|_1&\leq \Csys+\sum_{\tau=0}^{T-2}\|C_{i:}A^{\tau}B\|_1\nonumber\\
	&\leq \Csys+\sum_{\tau=0}^{\infty}\|C\|_1\|A^{\tau}\|_1\|B\|_1\nonumber\\
	&\leq \Csys+\sum_{\tau=0}^{\infty}\Csys^3\rho^\tau\nonumber\\
	&\leq \frac{2\Csys^3}{1-\rho}\nonumber
	\end{align}
	which completes the proof.
\end{proof}
Our next goal is to construct a sharp expression for $f(\Delta_{i:})$. As mentioned before, the matrix $U$ will have zero singular values when $N< Tp$. Therefore, the standard techniques for showing the concentration of the singular values of circulant matrices around an strictly positive number cannot be established. To circumvent this challenge, we prove the following key lemma which plays a pivotal role in our subsequent analysis.
\begin{lemma}\label{l_f}
	Suppose that $u_t$ is a zero-mean Gaussian vector with covariance $\sigma^2_u I$ for every $t = 1,\dots T+N-2$. Moreover, assume that $N\geq 4\eta\log^2(Tp)$ for an arbitrary $\eta>0$. Then, we have
	\begin{align}
	\frac{1}{N}\|U\theta\|_2^2\geq \frac{\sigma_u^2}{2}\|\theta\|_2^2-\sigma_u^2\sqrt\frac{\eta\log(Tp)}{{N}}\|\theta\|_1^2
	\end{align}
	for $\theta\in\mathbb{R}^{Tp}$, with probability of at least $1-(Tp)^{-c\eta}$.
\end{lemma}
\begin{proof}
	See Appendix~\ref{app_l_f}.
\end{proof}
Equipped with this lemma, we are now ready to present the appropriate choices of $\kappa$ and $f$ in Proposition~\ref{prop_deterministic}.
\begin{proposition}[Restricted singular value]\label{prop_f}
	Assume that $N\geq 4\eta\log^2(Tp)$ for an arbitrary $\eta>0$. Then, for any fixed row index $i$, the following inequality holds:
	\begin{align}
	\frac{1}{N}\|U\Delta_{i:}\|_2^2\geq \frac{\sigma_u^2}{4}\|\Delta_{i:}\|_2^2-128\sigma_u^2\left(\frac{\Csys^3}{1-\rho}\right)^2\sqrt{\frac{\eta\log(Tp)}{N}}
	\end{align}
	with probability of at least $1-(Tp)^{-c\eta}$.
\end{proposition}
\begin{proof}
	See Appendix~\ref{app_prop_f}.
\end{proof}
According to the above proposition, it is possible to choose $f(\Delta_{i:})$ such that it diminishes at the rate of $O(\sqrt{\log(Tp)/N})$ while $\kappa$ remains constant. 

Finally, we will provide a lower bound on $\lambda$ in terms of the system parameters, $T$, $p$, and $N$, to ensure that the third assumption of Proposition~\ref{prop_deterministic} holds with high probability. 
\begin{proposition}[Bound on $\lambda$]\label{prop_lambda}
	Suppose that $T$ satisfies:
	\begin{align}
	T\geq \frac{\log\log(Np+Tp+Nn)+4\log(\frac{\Csys}{1-\rho})+4\log(\sigma_w+\sigma_u)+2\log (2)}{1-\rho}+2
	\end{align}
	Then, for an arbitrary $\eta>0$, the following inequality holds
	\begin{multline}
	\frac{2}{N}\left(\|U^\top WF_{i:}^\top\|_\infty+\|U^\top E_{:i}\|_\infty+\|U^\top V_{:i}\|_\infty\right)\leq\\
	4\sqrt{2}\sigma_u\sigma_w\left(\frac{\Csys^2}{1-\rho}\right)\sqrt{(1+\eta)\frac{\log(Tpn)}{N}}+4\sigma_u\sigma_v\sqrt{(1+\eta)\frac{\log(Tp)}{N}}+2\rho^{T/2}(1+\eta)
	\end{multline}
	with the probability of at least $1-2(Nn)^{-\eta}-2(Np+Tp)^{-\eta}-2(Tp)^{-\eta}$. 
\end{proposition}
\begin{proof}
	See Appendix~\ref{app_prop_lambda}.
\end{proof}

\noindent{\it Proof of Theorem~\ref{thm_main}.} We provide the proof in four steps:
\begin{enumerate}
	\item According to Proposition~\ref{prop_G}, $\frac{\Csys^3}{1-\rho}$ is a valid choice for $R$ to satisfy the first assumption of Proposition~\ref{prop_deterministic}.
	\item Proposition~\ref{prop_f} implies that the second assumption of Proposition~\ref{prop_deterministic} holds with high probability, with $\kappa = \frac{\sigma_u^2}{4}$ and $f(\Delta_{i:})\asymp  \sigma_u^2\left(\frac{\Csys^3}{1-\rho}\right)^2\sqrt{\frac{\log(Tp)}{N}}$.
	\item Proposition~\ref{prop_lambda} shows that the third assumption of Proposition~\ref{prop_deterministic} holds with high probability with the choice of
	\begin{align}
	\lambda \asymp \left(\sigma_u\bar\sigma_w+\sigma_u\sigma_v\right)\sqrt{\frac{\log(Tpn)}{N}}+\epsilon
	\end{align}
	for an arbitrary $\epsilon>0$, provided that 
	\begin{align}
	T\gtrsim\frac{\log\log(Np+Tp+Nn)+\log(\frac{\Csys}{1-\rho})+\log(\sigma_u+\sigma_w)+\log(\frac{1}{\epsilon})}{1-\rho}
	\end{align}
	\item Finally, it is easy to verify that the following inequalities hold for every row index $i$:
	\begin{align}
	\frac{2}{\kappa}f(\Delta_{i:})&\lesssim \left(\frac{\Csys^3}{1-\rho}\right)^2\!\!\sqrt{\frac{\log(Tp)}{N}},\\
	\frac{88R}{\kappa^2}\lambda&\lesssim \left(\frac{\Csys^3}{1-\rho}\right)\left(\left(\frac{\bar\sigma_w+\sigma_v}{\sigma_u^3}\right)\sqrt{\frac{\log(Tpn)}{N}}+\frac{1}{\sigma_u^4}\epsilon\right)
	\end{align}
\end{enumerate}
These inequalities, combined with~\eqref{bound_det} and a simple union bound on different rows of $G$ proves the validity of~\eqref{bound_2inf}. Moreover,~\eqref{bound_F} follows from $\|G-\widehat{G}\|_F\leq\sqrt{m}\|G-\widehat{G}\|_{2,\infty}$.$\hfill\square$\vspace{2mm}

Next, we will present the proof of Corollary~\ref{cor_HM}.\vspace{2mm}

\noindent{\it Proof of Corollary~\ref{cor_HM}.} It is easy to see that
\begin{align}
\|\widehat{G}^{(K)}-{G}^{(K)}\|_{2,\infty}\leq \|\widehat{G}-{G}\|_{2,\infty}+\sum_{\tau=T-1}^{K-2}\|CA^{\tau}B\|_{2,\infty}
\end{align}
On the other hand, a simple application of the H\"older's inequality leads to
\begin{align}
\sum_{\tau=T-1}^{K-2}\|CA^{\tau}B\|_{2,\infty}\leq \sqrt{\|C\|_\infty}\Csys\sum_{\tau=T-1}^{K-2}\rho^{\tau/2}\leq \sqrt{\|C\|_\infty}\left(\frac{\Csys}{1-\rho}\right)\rho^{\frac{T-1}{2}}
\end{align}
The above expression is upper bounded by $\tilde{\epsilon}$, provided that
\begin{align}\label{T_lower}
T\geq\frac{\log(\|C\|_\infty)+2\log\left(\frac{\Csys}{1-\rho}\right)+2\log\left(\frac{1}{\tilde{\epsilon}}\right)}{1-\rho}+1
\end{align}
Combined with Theorem~\ref{thm_main}, this certifies the validity of~\eqref{GK_2inf}. The inequality~\eqref{GK_F} follows from $\|\widehat{G}^{(K)}-{G}^{(K)}\|_{F}\leq \sqrt{m}\|\widehat{G}^{(K)}-{G}^{(K)}\|_{2,\infty}$. Moreover, the correctness of~\eqref{HK_2inf} can verified by noting that the rows of $H^{(K)}-\widehat{H}^{(K)}$ are subvectors of the rows of $G^{(2K-1)}-\widehat G^{(2K-1)}$. Finally, to show the correctness of~\eqref{HK_F}, note that 
\begin{align}
\|H^{(K)}-\widehat H^{(K)}\|^2_F =& \|D-\widehat D\|_F^2+\sum_{k=0}^{T-2}(k+2)\|G_k-\widehat G_k\|_F^2+\sum_{k=T-1}^{K-2}(k+2)\|CA^kB\|_F^2\nonumber\\
&+\sum_{k=K-1}^{2K-3}(2K-2-k)\|CA^kB\|_F^2\nonumber\\
\leq& T\|G-\widehat{G}\|_F^2+m\|C\|_\infty\Csys^2\left(\sum_{k=T-1}^{K-2}(k+2)\rho^k+\sum_{k=K-1}^{2K-3}(2K-2-k)\rho^k\right)\nonumber\\
\leq&T\|G-\widehat{G}\|_F^2+m\|C\|_\infty\Csys^2\left(\frac{T\rho^{T-1}}{1-\rho}+\frac{\rho^{T-1}}{(1-\rho)^2}\right)\nonumber\\
\leq& Tm\left(\|G-\widehat{G}\|^2_{2,\infty}+2\|C\|_\infty\left(\frac{\Csys}{1-\rho}\right)^2\rho^{T-1}\right)
\end{align}
Therefore
\begin{align}
\|H^{(K)}-\widehat H^{(K)}\|_F&\leq \sqrt{Tm}\left(\|G-\widehat{G}\|_{2,\infty}+\sqrt{2\|C\|_\infty}\left(\frac{\Csys}{1-\rho}\right)\rho^{\frac{T-1}{2}}\right)\nonumber\\
&\leq \sqrt{Tm}\left(\|G-\widehat{G}\|_{2,\infty}+\sqrt{2}\tilde{\epsilon}\right)
\end{align}
where the second inequality follows from~\eqref{T_lower}. the above inequality combined with Theorem~\ref{thm_main} completes the proof.$\hfill\square$\vspace{2mm}

Finally, we will provide the proof for Corollary~\ref{cor_bound}.\vspace{2mm}

\noindent{\it Proof of Corollary~\ref{cor_bound}.} One can write
\begin{align}
\frac{\mathcal{E}^{\ell_1}_{F}}{\mathcal{E}^{LS}_{F}}\lesssim \underbrace{\left(\frac{N\log(Tpn)}{T^2(n+m+p)}\right)^{1/4}}_{(a)}+\underbrace{\left(\frac{N\epsilon^2}{T(n+m+p)}\right)^{1/2}}_{(b)}
\end{align}
It is easy to see that $(a)$ approaches zero with $(n,m,p)\to\infty$, due to the assumption~\eqref{eq_assume}. Moreover, due to Proposition~\ref{prop_lambda}, we have $\epsilon^2\lesssim\rho^{T}$. Therefore, $T\gtrsim \log(T/N)+\log(n+m+p)$ ensures that $(b)$ approaches zero with $(n,m,p)\to\infty$. The proof is completed by noting that $\log(T/N)\leq T/2$ for every $T,N\geq 1$.$\hfill\square$
\section{Proof of the Auxiliary Results}
\subsection{Proof of Proposition~\ref{prop_deterministic}}\label{app_prop1}
For the sake of simplicity, we suppress the row index $i$ and denote $Y_{:i}$, $X_{i:}^\top$, $G_{i:}^\top$, and $\widehat G_{i:}^\top$ as $y$, $g$, $g^*$, and $\hat{g}$, respectively. Therefore,~\eqref{lasso_i} can be written as
\begin{align}\label{eq_theta}
\hat{g} = \arg\min_{g\in\mathbb{R}^{Tp}}\frac{1}{2N}\|y-Ug\|^2_2+\lambda\|g\|_1
\end{align}
Furthermore, we treat the combined term $w = [WF^\top]_{:i} + E_{:i} + V_{:i}$ as the additive noise. Finally, the estimation error is denoted as ${\delta} = \hat{g}-g^*$. Note that $\hat{g}$ is an optimal solution of~\eqref{eq_theta}. Therefore, one can write
\begin{align}\label{eq_fundamental}
&\frac{1}{2N}\|y-U\hat{g}\|_2^2+\lambda\|\hat{g}\|_1\leq \frac{1}{2N}\|y-U{g}^*\|_2^2+\lambda\|{g}^*\|_1\nonumber\\
\implies& \frac{1}{2N}\|w-U\delta\|_2^2+\lambda\|\hat{g}\|_1\leq \frac{1}{2N}\|w\|_2^2+\lambda\|{g}^*\|_1\nonumber\\
\implies& \frac{1}{2N}\|U\delta\|_2^2\leq \frac{1}{N}w^\top U\delta+\lambda(\|{g}^*\|_1-\|\hat{g}\|_1)
\end{align}
On the other hand, given an arbitrary index set $\mathcal{S}\in\{1,\dots,Tp\}$ with $|\mathcal{S}| = s$, one can write
\begin{align}
\|g^*\|_1-\|\hat{g}\|_1\leq& \|g^*_{\mathcal{S}}\|_1+\|g^*_{\mathcal{S}^c}\|_1 - \|g^*_{\mathcal{S}}+{\delta}_{\mathcal{S}}\|_1-\|g^*_{\mathcal{S}^c}+{\delta}_{\mathcal{S}^c}\|_1\nonumber\\
\leq& \|{\delta}_{\mathcal{S}}\|_1-\|{\delta}_{\mathcal{S}^c}\|_1+2\|g^*_{\mathcal{S}^c}\|_1
\end{align}
where the second line is implied by triangle inequality. Substituting this inequality in~\eqref{eq_fundamental} leads to 
\begin{align}
\frac{1}{2N}\|U{\delta}\|_2^2\leq \frac{1}{N}\|w^\top U\|_\infty\|{\delta}\|_1+\lambda(\|{\delta}_{\mathcal{S}}\|_1-\|{\delta}_{\mathcal{S}^c}\|_1+2\|g^*_{\mathcal{S}^c}\|_1)
\end{align}
On the other hand, since $\lambda\geq 2\|w^\top U\|_\infty/N$, one can write
\begin{align}\label{eq1}
\frac{1}{N}\|U{\delta}\|_2^2\leq& \lambda(\|\delta\|_1+2\|{\delta}_{\mathcal{S}}\|_1-2\|{\delta}_{\mathcal{S}^c}\|_1+4\|g^*_{\mathcal{S}^c}\|_1)\nonumber\\
\leq&\lambda(3\|{\delta}_{\mathcal{S}}\|_1-\|{\delta}_{\mathcal{S}^c}\|_1+4\|g^*_{\mathcal{S}^c}\|_1)\nonumber\\
\leq&\lambda(3\sqrt{s}\|\delta\|_2+4\|g^*_{\mathcal{S}^c}\|_1)
\end{align}
where the last inequality is due to $\|\delta_{\mathcal{S}}\|_1\leq \sqrt{s}\|\delta_{\mathcal{S}}\|_2\leq \sqrt{s}\|\delta\|_2$. Now, we consider two cases. If we have $\|\delta\|_2^2\leq \frac{2}{\kappa}f(\delta)$, then the bound~\eqref{bound_det} trivially holds. Therefore, suppose $\|\delta\|_2^2> \frac{2}{\kappa}f(\delta)$. 
Combining this inequality with Assumption 2 and~\eqref{eq1} leads to
\begin{align}
&\kappa\|\delta\|_2^2-f(\delta)\leq \lambda(3\sqrt{s}\|\delta\|_2+4\|g^*_{\mathcal{S}^c}\|_1)\nonumber\\
\implies& \|\delta\|_2^2\leq \frac{2\lambda}{\kappa}(3\sqrt{s}\|\delta\|_2+4\|g^*_{\mathcal{S}^c}\|_1)
\end{align}
Notice that this is a quadratic inequality in terms of $\|\delta\|_2$. Bounding the roots of this quadratic inequality leads to
\begin{align}
\|\delta\|_2^2\leq \frac{72\lambda^2}{\kappa^2}s+\frac{16\lambda}{\kappa}\|g_{\mathcal{S}^c}\|_1
\end{align}
Therefore, the following inequality holds for all the values of $\|\delta\|_2^2$:
\begin{align}
\|\delta\|_2^2\leq\max\left\{\frac{2}{\kappa}f(\delta),\frac{72\lambda^2}{\kappa^2}s+\frac{16\lambda}{\kappa}\|g_{\mathcal{S}^c}\|_1\right\}
\end{align}
Now, it remains to show that the set $\mathcal{S}$ can be chosen such that $\frac{72\lambda^2}{\kappa^2}s+\frac{16\lambda}{\kappa}\|g_{\mathcal{S}^c}\|_1\leq \frac{88R}{\kappa^2}\lambda$. To show this, define $\mathcal{S} = \{i:|g^*_i|\geq \lambda\}$. Note that
\begin{align}
R\geq \sum_{i=1}^{Tp}|g^*_i|\geq \sum_{i\in\mathcal{S}}|g^*_i|\geq s\lambda
\end{align}
Therefore, we have $s\leq \frac{R}{\lambda}$. Furthermore, one can write
\begin{align}
\|g^*_{\mathcal{S}^c}\|_1 \leq R
\end{align}
This implies that
\begin{align}
\frac{72\lambda^2}{\kappa^2}s+\frac{16\lambda}{\kappa}\|g_{\mathcal{S}^c}\|_1\leq \frac{72R}{\kappa^2}\lambda+\frac{16R}{\kappa}\lambda\leq\frac{88R}{\kappa^2}\lambda
\end{align}
which completes the proof.$\hfill\square$

\subsection{Proof of Lemma~\ref{l_f}}\label{app_l_f}
To prove Lemma~\ref{l_f}, we need the following well-known result on the quadratic forms of sub-Gaussian random vectors.
\begin{theorem}[Hanson-Wright Inequality]\label{thm_HW}
	Let $w\in\mathbb{R}^d$ be a random vector with independent zero-mean sub-Gaussian elements with variance 1. Given a square and symmetric matrix $M$, we have $w^\top Mw\geq \mathbb{E}\{w^\top Mw\}-t$ with probability of at least $1-\exp\left(-c\min\left\{\frac{t^2}{\|M\|_F^2}, \frac{t}{\|M\|_2}\right\}\right)$
\end{theorem}
The main idea behind the proof of this lemma is as follows:
\begin{itemize}
	\item[1.] We reformulate $\|U\theta\|_2^2$ as a quadratic form $w^\top P(\theta)w$, where $P(\theta)\in\mathbb{R}^{N\times N}$ is only a function of $\theta$, and $w\in\mathbb{R}^N$ is a random vector with independent zero-mean sub-Gaussian elements.
	\item[2.] We provide upper bounds on $\|P(\theta)\|_F^2$ and $\|P(\theta)\|_2^2$ in terms of $\|\theta\|_2^2$ and $\|\theta\|_1^2$.
	\item[3.] Finally, we apply Hanson-Wright inequality to obtain the desired concentration bound.
\end{itemize}
\begin{lemma}\label{l_reform}
	The following statements hold:
	\begin{itemize}
		\item[-] $\frac{1}{\sigma_u^2}\|U\theta\|_2^2 \sim w^TP(\theta)w$, where $w\in\mathbb{R}^N$ is a random vector with independent zero-mean Gaussian elements, and $P(\theta)\in\mathbb{R}^{N\times N}$ is a symmetric matrix defined as 
		\begin{align}
		P_{ij}(\theta) = R(|i-j|), \forall i,j\in\{1,\dots, N\}^2,\quad\text{and}\quad R(\tau) = \sum_{k=1}^{(T-\tau)p}\theta_k\theta_{k+\tau p}
		\end{align}
		\item[-] $\frac{1}{\sigma_u^2}\mathbb{E}\{\|U\theta\|_2^2\} = N\|\theta\|_2^2$.
	\end{itemize}
\end{lemma}
\begin{proof}
	Upon defining $\zeta = U\theta$, one can verify that $\zeta$ has a zero-mean Gaussian distribution. Moreover, it is easy to see that 
	\begin{align}
	\mathbb{E}\{\zeta_i\zeta_j\} &= \mathbb{E}\left\{\left(\bar{u}_{T+(i-1)}^\top 
	\theta\right)\left(\bar{u}_{T+(j-1)}^\top 
	\theta\right)\right\}\nonumber\\
	&= \sum_{k=1}^{(T-|i-j|)p}\theta_k\theta_{k+p|i-j|} = R(|i-j|)
	\end{align}
	where in the second equality, we used the following facts:
	\begin{itemize}
		\item[-] $\mathbb{E}\{u_{t}(s)^2\} = 1$ for every $t\in\{1,\dots, T+N-2\}$ and $s\in\{1,\dots,p\}$.
		\item[-] $\mathbb{E}\{u_{t}(s_1)u_{t}(s_2)\} = 0$ for every $t\in\{1,\dots, T+N-2\}$ and $s_1,s_2\in\{1,\dots,p\}$ such that $s_1\not=s_2$.
		\item[-] $\mathbb{E}\{u_{t_1}(s_1)u_{t_2}(s_2)\} = 0$ for every $t_1,t_2\in\{1,\dots, T+N-2\}$ such that $t_1\not=t_2$ and $s_1,s_2\in\{1,\dots,p\}$.
	\end{itemize}
	This implies that $\mathbb{E}\{\zeta_i\zeta_j\} = \sigma_u^2P_{ij}(\theta)$, and hence, $\frac{1}{\sigma_u}U\theta \sim \mathcal{N}(0,P(\theta))$. Therefore, $\frac{1}{\sigma_u^2}\|U\theta\|_2^2$ has the same distribution as $w^\top P(\theta)w$, thereby completing the proof of the first statement. The second statement directly follows from the definition of $P(\theta)$ and the fact that $w$ has independent elements.
\end{proof}
Our next lemma provides an upper bound on the values of $\|P(\theta)\|_F^2$ and $\|P(\theta)\|_2^2$.
\begin{lemma}\label{l_norms}
	The following inequalities hold:
	\begin{itemize}
		\item[-] $\|P(\theta)\|_2\leq \|\theta\|_2^2+\|\theta\|_1^2$
		\item[-] $\|P(\theta)\|_F^2\leq N(\|\theta\|_2^2+\|\theta\|_1^2)^2$
	\end{itemize}
\end{lemma}
\begin{proof}
	Due to the Gershgorin circle theorem, one can write $|\|P(\theta)\|_2-\|\theta\|_2^2|\leq \sum_{\tau=1}^{N-1}|R(\tau)|$. This implies that $\|P(\theta)\|_2\leq \|\theta\|_2^2+\sum_{\tau=1}^{N-1}|R(\tau)|$. Define $\tilde{\theta}_{i} = \theta_{Tp+1-i}$ and $\tilde{R}$ as the convolution of $\theta$ and $\tilde{\theta}$, i.e., $ \tilde{R}(\tau)= \sum_{k=1}^{\tau-1}\theta_k\tilde{\theta}_{\tau-k}$ for every $\tau$. One can write
	\begin{align}
	R(\tau) = \sum_{k=1}^{(T-\tau)p}\theta_k\theta_{k+\tau p} = \sum_{k=1}^{(T-\tau)p}\theta_k\tilde\theta_{(T-\tau)p+1-k} = \tilde R((T-\tau)p+1)
	\end{align}
	Therefore, we have
	\begin{align}\label{eq_R}
	\sum_{\tau=1}^{N-1}|R(\tau)| = \sum_{\tau=1}^{N-1}|\tilde{R}((T-\tau)p+1)|\leq \|\tilde{R}\|_1 = \|\theta \ast \tilde{\theta}\|_1\leq \|\theta\|_1 \|\tilde{\theta}\|_1 = \|\theta\|_1^2
	\end{align}
	where the last inequality is due to Young's convolution rule. This concludes the proof of the first statement. The second statement follows from~\eqref{eq_R} and $\|P(\theta)\|_F^2\leq N\|P(\theta)\|_2^2$.
\end{proof}
\noindent{\it Proof of Lemma~\ref{l_f}.} Lemmas~\ref{l_reform} and~\ref{l_norms}, together with Theorem~\ref{thm_HW} imply that $$\frac{1}{\sigma_u^2N}\|U\theta\|_2^2\geq \|\theta\|_2^2-t$$ for any $t>0$, with probability of at least 
\begin{align}
1-\exp\left(-c\min\left\{\frac{Nt^2}{(\|\theta\|_2^2+\|\theta\|_1^2)^2}, \frac{Nt}{\|\theta\|_2^2+\|\theta\|_1^2}\right\}\right)
\end{align}
Upon choosing $t = \sqrt{\frac{\eta\log(Tp)}{{N}}}(\|\theta\|_2^2+\|\theta\|_1^2)$ and $N\geq 4\eta\log(Tp)$ for some $\eta>0$, we have
\begin{align}
\frac{1}{N}\|U\theta\|_2^2&\geq \sigma_u^2\left(1-\sqrt\frac{\eta\log(Tp)}{{N}}\right)\|\theta\|_2^2-\sigma_u^2\sqrt\frac{\eta\log(Tp)}{{N}}\|\theta\|_1^2\nonumber\\
&\geq \frac{\sigma_u^2}{2}\|\theta\|_2^2-\sigma_u^2\sqrt\frac{\eta\log(Tp)}{{N}}\|\theta\|_1^2
\end{align}
with probability of at least $1-(Tp)^{-c\eta}$. This completes the proof.$\hfill\square$

\subsection{Proof of Proposition~\ref{prop_f}}\label{app_prop_f}
For simplicity, we will borrow the notations $g^*$, $\hat{g}$, and $\delta$ from the proof of Proposition~\ref{prop_deterministic}. First note that, according to~\eqref{eq1}, $\delta = g^*-\hat{g}$ satisfies the following property
\begin{align}\label{eq_sc}
\|\delta_{\mathcal{S}^c}\|_1\leq 3\|\delta_{\mathcal{S}}\|_1 + 4\|\theta^*_{\mathcal{S}^c}\|_1
\end{align}
for any choice of $\mathcal{S}\in\{1,2,\dots Tp\}$ with $|\mathcal{S}| = s$. Define $\mathcal{S} = \{i:|g^*_i|\geq \gamma\}$ for a value of $\gamma$ to be defined later. Assuming that $\|g^*\|_1\leq R$, the inequality~\eqref{eq_sc} implies that 
\begin{align}
&\|\delta\|_1\leq 4\|\delta_{\mathcal{S}}\|_1 + 4\|\theta^*_{\mathcal{S}^c}\|_1\leq 4\sqrt{s}\|\delta\|_2+4\|\theta^*_{\mathcal{S}^c}\|_1\leq 4\sqrt{\frac{R}{\gamma}}\|\delta\|_2+4R
\end{align}
where the last inequality follows from $s\leq R/\gamma$ and $\|\theta^*_{\mathcal{S}^c}\|_1\leq R$. This leads to 
\begin{align}
	\|\delta\|_1^2\leq \frac{32R}{\gamma}\|\delta\|^2_2+32R^2
\end{align}
Combining the above inequality with Lemma~\ref{l_f} implies that the following inequalities holds with probability of at least $1-(Tp)^{-c\eta}$:
\begin{align}
\frac{1}{N}\|U\delta\|_2^2\geq \sigma_u^2\left(\frac{1}{2}-\frac{32R}{\gamma}\sqrt{\frac{\eta\log(Tp)}{N}}\right)\|\delta\|_2^2-32\sigma_u^2R^2\sqrt{\frac{\eta\log(Tp)}{N}}
\end{align}
Now, upon choosing $\gamma = 128R\left(\frac{\eta\log(Tp)}{N}\right)$, we get $\frac{32R}{\gamma}\sqrt{\frac{\eta\log(Tp)}{N}} = 1/4$, which results in
\begin{align}
\frac{1}{N}\|U\delta\|_2^2\geq\frac{\sigma_u^2}{4}\|\delta\|_2^2-32\sigma_u^2R^2\sqrt{\frac{\eta\log(Tp)}{N}}
\end{align}
Finally, Proposition~\ref{prop_G} can be invoked to show that $R\leq \frac{2\Csys^3}{1-\rho}$. This completes the proof.$\hfill\square$

\subsection{Proof of Proposition~\ref{prop_lambda}}\label{app_prop_lambda}
To prove this proposition, we divide the lower bound in three different terms:
\begin{align}
\lambda\geq \underbrace{\frac{2\|U^\top WF_{i:}^\top\|_\infty}{N}}_{(I)}+\underbrace{\frac{2\|U^\top E_{:i}\|_\infty}{N}}_{(II)}+\underbrace{\frac{2\|U^\top V_{:i}\|_\infty}{N}}_{(III)}
\end{align}
Next, we will provide concentration bounds on every term of the above inequality.
\begin{lemma}[Bounding $(I)$]\label{l_I}
	The following inequality holds:
	\begin{align}
	\frac{2\|U^\top WF_{i:}^\top\|_\infty}{N}\leq 4\sqrt{2}\sigma_u\sigma_w\left(\frac{\Csys^2}{1-\rho}\right)\sqrt{(1+\eta)\frac{\log(Tpn)}{N}}
	\end{align}
	with probability of at least $1-2(Tpn)^{-2\eta}$, for an arbitrary $\eta>0$
\end{lemma}
\begin{proof}
	One can write $\|U^\top WF_{i:}^\top\|_\infty\leq \|U^\top W\|_{\infty,\infty}\|F_{i:}\|_1$. We will bound each term on the right hand side separately. First, note that
	\begin{align}\label{F}
	\|F_{i:}\|_1 = \sum_{j=1}^{Tp}|F_{ij}| &= \sum_{\tau=0}^{T-2}\|C_{i:}\|_1\|A^\tau\|_1\leq \sum_{\tau=0}^{\infty}\Csys^2\rho^\tau\leq \frac{\Csys^2}{1-\rho}
	\end{align}
	Now, let us focus on $\|U^\top W\|_{\infty,\infty}$. We have $\|U^\top W\|_{\infty,\infty} = \max_{i,j}\{|U_{:i}^\top W_{:j}|\}$. Note that the vectors $U_{:i}$ and $W_{:j}$ are independent random vectors, each with sub-Gaussian elements. Therefore, $U_{ki}^\top W_{kj}$ is sub-exponential with $(\sqrt{2}\sigma_u\sigma_w, \sqrt{2}\sigma_u\sigma_w)$~\cite{wainwright2019high} for every $k$. A standard concentration bound on sub-exponential random variables entails that
	\begin{align}\label{eq62}
	\mathbb{P}(|U_{:i}^\top W_{:j}|\leq \sqrt{2}\sigma_u\sigma_wt)\geq 1-2\exp\left(-\frac{1}{2}\min\left\{t,\frac{t^2}{N}\right\}\right)
	\end{align}
	A simple union bound implies that
	\begin{align}\label{eq63}
	\mathbb{P}\left(\frac{1}{N}\|U^\top W\|_{\infty, \infty}\leq \sqrt{2}\sigma_u\sigma_wt\right)\geq 1-2T^2pn\exp\left(-\frac{N}{2}\min\left\{t,t^2\right\}\right)
	\end{align}
	Now, define $t = c\sqrt{\frac{\log(Tpn)}{N}}$ for a constant $c$ to be defined later, and assume that $N\geq c^2\log(Tpn)$. This implies that $t^2\leq t$, which leads to
	\begin{align}
	\mathbb{P}\left(\frac{1}{N}\|U^\top W\|_{\infty, \infty}\leq \sqrt{2}c\sigma_u\sigma_w\sqrt{\frac{\log(Tpn)}{N}}\right)\geq 1-2\exp\left(2\log(Tpn)-\frac{c^2}{2}\log(Tpn)\right)
	\end{align}
	Now, upon defining $c = 2\sqrt{1+\eta}$ for an arbitrary $\eta>0$, we have
	\begin{align}
	\frac{1}{N}\|U^\top W\|_{\infty, \infty}\leq 2\sqrt{2}\sigma_u\sigma_w\sqrt{(1+\eta)\frac{\log(Tpn)}{N}}
	\end{align}
	with probability of at least $1-2(Tpn)^{-2\eta}$. Combining this inequality with~\eqref{F} completes the proof.
\end{proof}

\begin{lemma}[Bounding $(II)$]\label{l_II}
	Assume that
	\begin{align}
	T\geq \frac{\log\log(Np+Tp+Nn)+4\log(\frac{\Csys}{1-\rho})+4\log(\sigma_w+\sigma_u)+2\log (2)}{1-\rho}+2
	\end{align}
	Then, the following inequality holds
	\begin{align}
	\frac{2\|U^\top E_{:i}\|_\infty}{N}\leq 2\rho^{T/2}(1+\eta)
	\end{align}
	with probability of at least $1-2(Nn)^{-\eta}-2(Np+Tp)^{-\eta}$, for an arbitrary $\eta>0$.
\end{lemma}
\begin{proof}
	\begin{sloppypar}
		One can write $\|U^\top E_{:i}\|_\infty\leq \|U\|_{\infty,\infty}\|E_{:i}\|_1$. On the other hand, note that $E = \begin{bmatrix}
		e_{T-1} & e_{T} & \dots & e_{T+N-2}
		\end{bmatrix}^\top$, where $e_t = CA^{T-1}x_{t-T+1}$. This implies that
	\end{sloppypar}
	\begin{align}\label{eq65}
	\|E_{:i}\|_1 = \sum_{t=T-1}^{T+N-2}|(e_t)_i| = \sum_{t=1}^{N}|(CA^{T-1})_{i:}x_t|&\leq \|(CA^{T-1})_{i:}\|_1\sum_{t=1}^{N}\|x_t\|_\infty\nonumber\\
	&\leq N\|(CA^{T-1})_{i:}\|_1\|X\|_{\infty, \infty}
	\end{align}
	where $X = \begin{bmatrix}
	x_1 & x_2 &\dots & x_N
	\end{bmatrix}$. Similar to the previous case, we bound each term on the right hand side separately. First note that
	\begin{align}\label{eq66}
	\|(CA^{T-1})_{i:}\|_1\leq\|C_{i:}\|_1\|A^{T-1}\|_1\leq  \Csys^2\rho^{T-1}
	\end{align}
	Our next goal is to provide an upper bound on $\|X\|_{\infty, \infty}$. It is easy to see that, for every $t$, the vector $x_t$ is a zero-mean Gaussian variable with covariance
	\begin{align}
	\Sigma_t = \sum_{i=0}^{t-1}\sigma_w^2A^i(A^\top)^i+\sigma_u^2A^iBB^\top(A^\top)^i
	\end{align}
	This implies that each element $X_{ij}$ is Gaussian with variance $\sigma^2_{ij}$ satisfying
	\begin{align}
	\sigma^2_{ij}&\leq \sigma_w^2\sum_{i=0}^{\infty}\Csys^2\rho^{2i}+\sigma_u^2\sum_{i=0}^{\infty}\Csys^4\rho^{2i}\nonumber\\
	&\leq (\sigma_w^2+\sigma_u^2)\frac{\Csys^4}{1-\rho^2}
	\end{align}
	which implies that $\sigma_{ij}\leq (\sigma_w+\sigma_u)\frac{\Csys^2}{1-\rho}$.
	This together with the standard concentration bounds on the Gaussian distributions implies that
	\begin{align}
	\mathbb{P}\left(\|X\|_{\infty,\infty}\leq \left((\sigma_w+\sigma_u)\frac{\Csys^2}{1-\rho}\right)t\right) \geq 1-2Nn\exp\left(-\frac{t^2}{2}\right)
	\end{align} 
	Upon choosing $t = \sqrt{2(1+\eta)\log(Nn)}$, we have
	\begin{align}
	\|X\|_{\infty,\infty}\leq \sqrt{2}(\sigma_w+\sigma_u)\left(\frac{\Csys^2}{1-\rho}\right)\sqrt{(1+\eta)\log(Nn)}
	\end{align}
	with probability of at least $1-2(Nn)^{-\eta}$. Combining this inequality with~\eqref{eq65} and~\eqref{eq66} leads to 
	\begin{align}\label{eq71}
	\|E_{:i}\|_1&\leq N\Csys^2\rho^{T-1}(\sigma_w+\sigma_u)\left(\frac{\Csys^2}{1-\rho}\right)\sqrt{2(1+\eta)\log(Nn)}\nonumber\\
	&=N(\sigma_w+\sigma_u)\left(\frac{\Csys^4}{1-\rho}\right)\rho^{T-1}\sqrt{2(1+\eta)\log(Nn)}
	\end{align}
	with probability of at least $1-2(Nn)^{-\eta}$. Similarly, one can write 
	\begin{align}
	\mathbb{P}\left(\|U\|_{\infty,\infty}\leq \sigma_ut\right)\geq 1-2(N+T)p\exp\left(-\frac{t^2}{2}\right)
	\end{align}
	Upon choosing $t = \sqrt{2(1+\eta)\log(Np+Tp)}$, we have
	\begin{align}\label{eq72}
	\|U\|_{\infty,\infty}\leq\sigma_u\sqrt{2(1+\eta)\log(Np+Tp)}
	\end{align}
	with probability of at least $1-2(Np+Tp)^{-\eta}$, for some $\eta>0$. Combining all the derived bounds, one can write
	\begin{align}\label{eq722}
	\frac{1}{N}\|U^\top E_{:i}\|_\infty\leq 2(\sigma_w+\sigma_u)\sigma_u\left(\frac{\Csys^4}{1-\rho}\right){(1+\eta)\log(Np+Tp+Nn)}\rho^{T-1}
	\end{align}
	Now, if choose
	\begin{align}
	T\geq \frac{4\log(\frac{\Csys}{1-\rho})+4\log(\sigma_w+\sigma_u)+2\log (2)+2\log\log(Np+Tp+Nn)}{1-\rho}+2
	\end{align}
	then, we have
	\begin{align}
	\rho^{-(T/2-1)}\geq 2(\sigma_w+\sigma_u)\sigma_u\left(\frac{\Csys^4}{1-\rho}\right){\log(Np+Tp+Nn)}
	\end{align}
	Combining the above inequality with~\eqref{eq722} leads to
	\begin{align}
	\frac{1}{N}\|U^\top E_{:i}\|_\infty\leq \rho^{T/2}(1+\eta)
	\end{align}
	which holds with probability of at least $1-2(Nn)^{-\eta}-2(Np+Tp)^{-\eta}$. This completes the proof.
\end{proof}
\begin{lemma}[Bounding $(III)$]\label{l_III}
	The following inequality holds:
	\begin{align}
	\frac{2\|U^\top V_{:i}\|_\infty}{N}\leq 4\sigma_u\sigma_v\sqrt{(1+\eta)\frac{\log(Tp)}{N}}
	\end{align}
	with probability of at least $1-2(Tp)^{-\eta}$, for an arbitrary $\eta>0$.
\end{lemma}
\begin{proof}
	The proof is a simpler version of the proof of Lemma~\ref{l_I}, and the details are omitted for brevity.
\end{proof}

\noindent{\it Proof of Proposition~\ref{prop_lambda}.} The proof follows by combining the bounds obtained in Lemmas~\ref{l_I},~\ref{l_II},~\ref{l_III}.$\hfill\square$

\bibliographystyle{elsarticle-num}
\bibliography{reference.bib}

\begin{thebibliography}{10}
\expandafter\ifx\csname url\endcsname\relax
  \def\url#1{\texttt{#1}}\fi
\expandafter\ifx\csname urlprefix\endcsname\relax\def\urlprefix{URL }\fi
\expandafter\ifx\csname href\endcsname\relax
  \def\href#1#2{#2} \def\path#1{#1}\fi

\bibitem{blaabjerg2006overview}
F.~Blaabjerg, R.~Teodorescu, M.~Liserre, A.~V. Timbus, Overview of control and
  grid synchronization for distributed power generation systems, IEEE
  Transactions on industrial electronics 53~(5) (2006) 1398--1409.

\bibitem{amin2005toward}
S.~M. Amin, B.~F. Wollenberg, Toward a smart grid: power delivery for the 21st
  century, IEEE power and energy magazine 3~(5) (2005) 34--41.

\bibitem{wang1998robust}
Y.~Wang, D.~J. Hill, G.~Guo, Robust decentralized control for multimachine
  power systems, IEEE Transactions on Circuits and Systems I: Fundamental
  Theory and Applications 45~(3) (1998) 271--279.

\bibitem{barbaresso2014usdot}
J.~Barbaresso, G.~Cordahi, D.~Garcia, C.~Hill, A.~Jendzejec, K.~Wright, B.~A.
  Hamilton, Usdot’s intelligent transportation systems (its) its strategic
  plan, 2015-2019., Tech. rep., United States. Department of Transportation.
  Intelligent Transportation (2014).

\bibitem{krechmer2018effects}
D.~Krechmer, E.~Flanigan, A.~Rivadeneyra, K.~Blizzard, S.~Van~Hecke, R.~Rausch,
  Effects on intelligent transportation systems planning and deployment in a
  connected vehicle environment, Tech. rep., United States. Federal Highway
  Administration. Office of Operations (2018).

\bibitem{kapila2000spacecraft}
V.~Kapila, A.~G. Sparks, J.~M. Buffington, Q.~Yan, Spacecraft formation flying:
  Dynamics and control, Journal of Guidance, Control, and Dynamics 23~(3)
  (2000) 561--564.

\bibitem{kubisch2003distributed}
M.~Kubisch, H.~Karl, A.~Wolisz, L.~C. Zhong, J.~Rabaey, Distributed algorithms
  for transmission power control in wireless sensor networks, in: 2003 IEEE
  Wireless Communications and Networking, 2003. WCNC 2003., Vol.~1, IEEE, 2003,
  pp. 558--563.

\bibitem{nguyen2011model}
D.~Nguyen-Tuong, J.~Peters, Model learning for robot control: a survey,
  Cognitive processing 12~(4) (2011) 319--340.

\bibitem{sutton2018reinforcement}
R.~S. Sutton, A.~G. Barto, Reinforcement learning: An introduction, MIT press,
  2018.

\bibitem{krizhevsky2012imagenet}
A.~Krizhevsky, I.~Sutskever, G.~E. Hinton, Imagenet classification with deep
  convolutional neural networks, in: Advances in neural information processing
  systems, 2012, pp. 1097--1105.

\bibitem{duan2016benchmarking}
Y.~Duan, X.~Chen, R.~Houthooft, J.~Schulman, P.~Abbeel, Benchmarking deep
  reinforcement learning for continuous control, in: International Conference
  on Machine Learning, 2016, pp. 1329--1338.

\bibitem{oymak2019non}
S.~Oymak, N.~Ozay, Non-asymptotic identification of lti systems from a single
  trajectory, in: 2019 American Control Conference (ACC), IEEE, 2019, pp.
  5655--5661.

\bibitem{krauth2019finite}
K.~Krauth, S.~Tu, B.~Recht, Finite-time analysis of approximate policy
  iteration for the linear quadratic regulator, in: Advances in Neural
  Information Processing Systems, 2019, pp. 8512--8522.

\bibitem{dean2019sample}
S.~Dean, H.~Mania, N.~Matni, B.~Recht, S.~Tu, On the sample complexity of the
  linear quadratic regulator, Foundations of Computational Mathematics (2019)
  1--47.

\bibitem{simchowitz2019learning}
M.~Simchowitz, R.~Boczar, B.~Recht, Learning linear dynamical systems with
  semi-parametric least squares, arXiv preprint arXiv:1902.00768 (2019).

\bibitem{sarkar2019finite}
T.~Sarkar, A.~Rakhlin, M.~A. Dahleh, Finite-time system identification for
  partially observed lti systems of unknown order, arXiv preprint
  arXiv:1902.01848 (2019).

\bibitem{aastrom1971system}
K.~J. {\AA}str{\"o}m, P.~Eykhoff, System identification—a survey, Automatica
  7~(2) (1971) 123--162.

\bibitem{ljung1999system}
L.~Ljung, System identification, Wiley encyclopedia of electrical and
  electronics engineering (1999) 1--19.

\bibitem{chen2012identification}
H.-F. Chen, L.~Guo, Identification and stochastic adaptive control, Springer
  Science \& Business Media, 2012.

\bibitem{goodwin1977dynamic}
G.~C. Goodwin, R.~Payne, Dynamic system identification. experiment design and
  data analysis. (1977).

\bibitem{dean2017sample}
S.~Dean, H.~Mania, N.~Matni, B.~Recht, S.~Tu, On the sample complexity of the
  linear quadratic regulator, arXiv preprint arXiv:1710.01688 (2017).

\bibitem{dean2018regret}
S.~Dean, H.~Mania, N.~Matni, B.~Recht, S.~Tu, Regret bounds for robust adaptive
  control of the linear quadratic regulator, in: Advances in Neural Information
  Processing Systems, 2018, pp. 4188--4197.

\bibitem{simchowitz2018learning}
M.~Simchowitz, H.~Mania, S.~Tu, M.~I. Jordan, B.~Recht, Learning without
  mixing: Towards a sharp analysis of linear system identification, arXiv
  preprint arXiv:1802.08334 (2018).

\bibitem{sarkar2018fast}
T.~Sarkar, A.~Rakhlin, How fast can linear dynamical systems be learned?, arXiv
  preprint arXiv:1812.01251 (2018).

\bibitem{tsiamis2019finite}
A.~Tsiamis, G.~J. Pappas, Finite sample analysis of stochastic system
  identification, arXiv preprint arXiv:1903.09122 (2019).

\bibitem{zheng2020non}
Y.~Zheng, N.~Li, Non-asymptotic identification of linear dynamical systems
  using multiple trajectories, arXiv preprint arXiv:2009.00739 (2020).

\bibitem{fattahi2019learning}
S.~Fattahi, N.~Matni, S.~Sojoudi, Learning sparse dynamical systems from a
  single sample trajectory, arXiv (2019) arXiv--1904.

\bibitem{fattahi2018data}
S.~Fattahi, S.~Sojoudi, Data-driven sparse system identification, in: 2018 56th
  Annual Allerton Conference on Communication, Control, and Computing
  (Allerton), IEEE, 2018, pp. 462--469.

\bibitem{fattahi2018non}
S.~Fattahi, S.~Sojoudi, Non-asymptotic analysis of block-regularized regression
  problem, in: 2018 IEEE Conference on Decision and Control (CDC), IEEE, 2018,
  pp. 27--34.

\bibitem{fattahi2018sample}
S.~Fattahi, S.~Sojoudi, Sample complexity of sparse system identification
  problem, arXiv preprint arXiv:1803.07753 (2018).

\bibitem{sun2020finite}
Y.~Sun, S.~Oymak, M.~Fazel, Finite sample system identification: Improved rates
  and the role of regularization (2020).

\bibitem{wahlberg2013matrix}
B.~Wahlberg, C.~Rojas, Matrix rank optimization problems in system
  identification via nuclear norm mimization, in: (Tutorial) at the European
  Control Conference, 2013.

\bibitem{cai2016robust}
J.-F. Cai, X.~Qu, W.~Xu, G.-B. Ye, Robust recovery of complex exponential
  signals from random gaussian projections via low rank hankel matrix
  reconstruction, Applied and computational harmonic analysis 41~(2) (2016)
  470--490.

\bibitem{abbasi2011regret}
Y.~Abbasi-Yadkori, C.~Szepesv{\'a}ri, Regret bounds for the adaptive control of
  linear quadratic systems, in: Proceedings of the 24th Annual Conference on
  Learning Theory, 2011, pp. 1--26.

\bibitem{abbasi2019model}
Y.~Abbasi-Yadkori, N.~Lazic, C.~Szepesv{\'a}ri, Model-free linear quadratic
  control via reduction to expert prediction, in: The 22nd International
  Conference on Artificial Intelligence and Statistics, 2019, pp. 3108--3117.

\bibitem{lale2020regret}
S.~Lale, K.~Azizzadenesheli, B.~Hassibi, A.~Anandkumar, Regret bound of
  adaptive control in linear quadratic gaussian (lqg) systems, arXiv preprint
  arXiv:2003.05999 (2020).

\bibitem{dean2019safely}
S.~Dean, S.~Tu, N.~Matni, B.~Recht, Safely learning to control the constrained
  linear quadratic regulator, in: 2019 American Control Conference (ACC), IEEE,
  2019, pp. 5582--5588.

\bibitem{mania2019certainty}
H.~Mania, S.~Tu, B.~Recht, Certainty equivalent control of lqr is efficient,
  arXiv preprint arXiv:1902.07826 (2019).

\bibitem{fattahi2019efficient}
S.~Fattahi, N.~Matni, S.~Sojoudi, Efficient learning of distributed
  linear-quadratic controllers (2019).
\newblock \href {http://arxiv.org/abs/1909.09895} {\path{arXiv:1909.09895}}.

\bibitem{furieri2020learning}
L.~Furieri, Y.~Zheng, M.~Kamgarpour, Learning the globally optimal distributed
  lq regulator, in: Learning for Dynamics and Control, 2020, pp. 287--297.

\bibitem{ho1966effective}
B.~Ho, R.~E. K{\'a}lm{\'a}n, Effective construction of linear state-variable
  models from input/output functions, at-Automatisierungstechnik 14~(1-12)
  (1966) 545--548.

\bibitem{antoulas2005approximation}
A.~C. Antoulas, Approximation of large-scale dynamical systems, SIAM, 2005.

\bibitem{zhou1998essentials}
K.~Zhou, J.~C. Doyle, Essentials of robust control, Vol. 104, Prentice hall
  Upper Saddle River, NJ, 1998.

\bibitem{wainwright2019high}
M.~J. Wainwright, High-dimensional statistics: A non-asymptotic viewpoint,
  Vol.~48, Cambridge University Press, 2019.

\bibitem{jin2017sparse}
K.~H. Jin, J.~C. Ye, Sparse and low-rank decomposition of a hankel structured
  matrix for impulse noise removal, IEEE Transactions on Image Processing
  27~(3) (2017) 1448--1461.

\bibitem{tu2017non}
S.~Tu, R.~Boczar, A.~Packard, B.~Recht, Non-asymptotic analysis of robust
  control from coarse-grained identification, arXiv preprint arXiv:1707.04791
  (2017).

\bibitem{negahban2012unified}
S.~N. Negahban, P.~Ravikumar, M.~J. Wainwright, B.~Yu, et~al., A unified
  framework for high-dimensional analysis of $ m $-estimators with decomposable
  regularizers, Statistical science 27~(4) (2012) 538--557.

\bibitem{shao1993linear}
J.~Shao, Linear model selection by cross-validation, Journal of the American
  statistical Association 88~(422) (1993) 486--494.

\bibitem{krahmer2014suprema}
F.~Krahmer, S.~Mendelson, H.~Rauhut, Suprema of chaos processes and the
  restricted isometry property, Communications on Pure and Applied Mathematics
  67~(11) (2014) 1877--1904.

\end{thebibliography}
\end{document}